\documentclass{article} 
\usepackage{iclr2026_conference,times}

\usepackage[T1]{fontenc}
\usepackage[utf8]{inputenc}
\usepackage{hyperref}
\usepackage{url}
\usepackage{amsmath,amssymb,amsthm}
\usepackage{bm}
\usepackage{booktabs}
\usepackage{array}
\usepackage{longtable}
\usepackage{enumitem}
\usepackage{graphicx}
\usepackage{microtype}
\usepackage{titletoc}
\usepackage{multirow}

\newtheorem{assumption}{Assumption}
\newtheorem{lemma}{Lemma}
\newtheorem{theorem}{Theorem}
\newtheorem{proposition}{Proposition}
\newtheorem{corollary}{Corollary}
\newtheorem{remark}{Remark}
\newtheorem{definition}{Definition}

\newcommand{\R}{\mathbb{R}}                    
\newcommand{\E}{\mathbb{E}}                    

\newcommand{\one}{\mathbf{1}}
\newcommand{\inner}[2]{\langle #1, #2\rangle}  
\newcommand{\Mcal}{\mathcal{M}}                
\newcommand{\GL}{\mathrm{GL}}                  
\newcommand{\Od}{\mathrm{O}}                   
\newcommand{\Perm}{\mathrm{Perm}}              
\newcommand{\Dplus}{\mathrm{D}_{+}}            
\newcommand{\Dpm}{\mathrm{D}_{\pm}}            
\DeclareMathOperator{\tr}{tr}
\DeclareMathOperator{\rank}{rank}
\newcommand{\irank}{\mathrm{i}\text{-}\mathrm{rank}}  
\newcommand{\diag}{\operatorname{diag}}
\newcommand{\offdiag}{\operatorname{offdiag}}

\title{Low-Interaction-Rank Learning: \\ Unifying Multiplicative Dual-Encoder Heads}

\author{Zijian Zhao$^1$, Sen Li$^{1,2}$ \thanks{Corresponding Author: Sen Li} \\
$^1$The Hong Kong University of Science and Technology \\
$^2$The Hong Kong University of Science and Technology (Guangzhou)
}
\begin{document}
\maketitle

\begin{abstract}
A multiplicative dual-encoder network computes a real-valued output for a pair of inputs as the inner product of their separate encodings. This architecture has been developed independently in operator learning, bipartite matching, contrastive vision-language models, retrieval, and other areas, yet no unified theory guides the basic design decisions: how many interaction modes to represent, how to normalize the encoders, and when the architecture should be avoided.
We provide such a foundation by introducing the class of functions of low interaction rank, a class whose intrinsic complexity is measured by its interaction spectrum. Within this framework, approximation error decomposes into a spectral truncation term and an encoder-realization term; sample complexity is governed by the sum of the two encoder complexities rather than their product; and a usability criterion based on spectral decay determines when the architecture can succeed. The same framework exposes a central identifiability problem: the encoders are defined only up to a linear gauge symmetry that leaves the learned coordinates arbitrary. We show that normalization is gauge fixing and that whitening pins the interaction modes up to permutation and sign, thereby explaining the uninterpretability of contrastive dimensions and providing a constructive remedy.
Experiments on synthetic kernels, operator learning, and CLIP models validate the theoretical predictions: spectral decay rates match the predicted scaling, whitening recovers the true modes, and independently trained CLIP models are related by a single rotation which, after removal by whitening, exposes interpretable concept axes. The code of this paper is provided at \url{https://github.com/RS2002/Mul-Net}.
\end{abstract}

\section{Introduction}
\label{sec:intro}

A common architectural pattern has emerged independently across several machine learning communities: computing a real-valued output for a pair of inputs as the inner product of their separately encoded representations. Contrastive models use it for cross-modal similarity~\citep{radford2021clip,zhao2025crossfi}; operator networks evaluate learned maps via branch-trunk inner products~\citep{lu2021deeponet}; retrieval systems rank queries against documents with two-tower encoders~\citep{karpukhin2020dpr}; goal-conditioned reinforcement learning and successor features represent values as inner products of encodings~\citep{schaul2015uvfa,hong2022bvn,barreto2017sf}; and knowledge-graph completion, linear attention, factorization machines, and multiplicative matching networks all share the same underlying structure~\citep{yang2015distmult,trouillon2016complex,katharopoulos2020linearattn,rendle2010fm,triplebert2026,assignmentnet2025}. 

In each domain, this structure provides tangible computational and representational benefits. Yet the corresponding methods have been developed largely in isolation, even though they are all instances of the same fundamental architecture, differing only in the nature of the two inputs and how the output is consumed. Consequently, each community independently confronts the same design questions: how many interaction modes to retain, how to normalize the encoders to ensure stable training and meaningful representations, and when the architecture is fundamentally ill-suited.

We now formalize this shared architecture: a multiplicative dual-encoder head computes
\begin{equation}
\boxed{F(u,v)\approx\inner{f_{\theta}(u)}{g_{\varphi}(v)}=\sum_{k=1}^{d}f_{k}(u)\,g_{k}(v)}
\label{eq:head}
\end{equation}
where $f_{\theta}$ and $g_{\varphi}$ are learned encoders mapping the two inputs into a common space $\R^{d}$. We study the class
\begin{equation}
\Mcal_d=\{\,F:\,F(u,v)=\inner{f(u)}{g(v)},\; f:U\to\R^{d},\; g:V\to\R^{d}\,\}
\label{eq:class}
\end{equation}
of functions representable by a rank-$d$ head. The complexity of a target $F$ within this class is measured by its interaction spectrum $\{\sigma_k\}_{k\ge 1}$, the singular values of the integral operator whose kernel is $F$; its interaction rank is the number of nonzero singular values.

Within this framework, we answer four questions about approximation error, encoder identifiability, sample complexity, and fundamental architectural limitations.
\begin{itemize}[leftmargin=1.4em,itemsep=2pt,topsep=2pt]
\item \textbf{Approximation:} The error of any rank-$d$ head decomposes into a spectral truncation term, unavoidable by any encoder, and an encoder-realization term; target smoothness controls the decay rate. (Section~\ref{sec:unif})
\item \textbf{Identifiability:} The representation~\eqref{eq:head} is invariant under a linear gauge symmetry acting on the pair of encoders, so the encoders of any trained model are defined only up to this symmetry. We show that the normalizations used across these domains are precisely gauge-fixing choices, and we prove that whitening pins the interaction modes up to permutation and sign. This explains why the individual dimensions of contrastive models lack inherent semantic meaning and provides a constructive remedy. (Section~\ref{sec:gauge})
\item \textbf{Estimation:} Sample complexity is governed by the sum of the two encoder complexities rather than their product; for smooth targets, the optimal rank grows only logarithmically in the sample budget. (Section~\ref{sec:sample})
\item \textbf{Usability:} A flat interaction spectrum forces every rank-$d$ head to a relative error floor of $1-d/N$, where $N$ is the size of the discrete input domain; this barrier is escaped by early-interaction networks and yields a practical criterion based on the measured spectrum. (Section~\ref{sec:sample})
\end{itemize}
To further validate our conclusions, we conducted experiments on synthetic kernels, operator learning, and CLIP models, confirming the predicted spectral behavior, demonstrating unique mode recovery under whitening, and verifying the flat-spectrum error floor.  (Section~\ref{sec:exp})

\section{The Low-Interaction-Rank Class: Unification and Approximation}
\label{sec:unif}

Section~\ref{sec:intro} introduced the class $\Mcal_d$ of rank-$d$ heads. We now associate with every target a spectral object that measures its intrinsic complexity within this class, show that ten existing method families are instances of the same architecture, and develop an approximation theory that quantifies how well a rank-$d$ head can represent a given target.

\subsection{The class and the interaction spectrum}

Let $\mathcal{U}$ and $\mathcal{V}$ be compact metric spaces with probability measures $\mu_U$, $\mu_V$, and let $F\in L^2(\mu_U\otimes\mu_V)$ be the target function of two arguments. A rank-$d$ head is a pair of encoders $f:\mathcal{U}\to\R^d$, $g:\mathcal{V}\to\R^d$ with coordinate functions $f_k\in L^2(\mu_U)$, $g_k\in L^2(\mu_V)$.

\begin{definition}[Interaction spectrum]
\label{def:spectrum}
For $F\in L^2(\mu_U\otimes\mu_V)$, the \emph{interaction operator} is the integral operator $T_F:L^2(\mu_V)\to L^2(\mu_U)$ defined by
\begin{equation}
(T_F h)(u)=\int_{\mathcal{V}}F(u,v)\,h(v)\,d\mu_V(v).
\end{equation}
Since $F\in L^2(\mu_U\otimes\mu_V)$, this operator is Hilbert-Schmidt, and $F$ admits the Schmidt decomposition
\begin{equation}
F=\sum_{k\ge 1}\sigma_k\,a_k\otimes b_k,\qquad (a_k\otimes b_k)(u,v)=a_k(u)\,b_k(v),
\end{equation}
with singular values $\sigma_1\ge\sigma_2\ge\cdots\ge 0$ and orthonormal systems $\{a_k\}\subset L^2(\mu_U)$, $\{b_k\}\subset L^2(\mu_V)$. The sequence $\{\sigma_k\}$ is the \emph{interaction spectrum} of $F$; the number $\irank(F)=\#\{k:\sigma_k>0\}$ is its \emph{interaction rank}; and the functions $a_k$, $b_k$ are its \emph{interaction modes}. The class introduced in Section~\ref{sec:intro} satisfies $\Mcal_d=\{F:\irank(F)\le d\}$, and the spectrum depends on the reference measures.
\end{definition}

By the isometry between $L^2(\mu_U\otimes\mu_V)$ and the Hilbert-Schmidt class, a function is representable by a rank-$d$ head if and only if $\rank(T_S)\le d$; hence the parametric class~\eqref{eq:class} coincides with $\Mcal_d$ (Lemma~\ref{lem:rank-equiv}, Appendix~\ref{app:proofs-sec2}). The spectral truncation term $\sum_{k>d}\sigma_k^2$ depends only on the target and the reference measures, not on the encoders: no head, regardless of its capacity, can represent $F$ with error below this quantity. The bound is attained in $L^2$ by the truncated Schmidt series, and is unique when $\sigma_d>\sigma_{d+1}$ (Corollary~\ref{cor:trunc-lower}, Appendix~\ref{app:proofs-sec2}). The decay rate of this tail, which is controlled by the smoothness of $F$ (Theorem~\ref{thm:decay}), determines how many interaction modes $d$ a good head must retain.

This is the operator-theoretic analogue of classical low-rank matrix approximation: the Schmidt decomposition is the infinite-dimensional Eckart-Young theorem~\citep{schmidt1907theorie,eckart1936approximation}, and the truncation term is exactly the Hilbert-Schmidt error of the best rank-$d$ approximation of $T_F$. The spectrum is relative to the reference measures $\mu_U,\mu_V$, which fix the inner product in which the modes are orthonormal; the same target may have different spectra under different measures. All subsequent claims therefore hold for the measures under which a head is trained and evaluated.

\subsection{Ten families as instances of the head}

Appendix~\ref{app:unif} collects ten method families from different communities in Table~\ref{tab:unif}. Each computes a real-valued output for a pair of inputs as the inner product of two separately learned encodings. The questions of how many interaction modes $d$ to keep, how to normalize the two encoders (Section~\ref{sec:gauge}), and how many samples the head requires (Section~\ref{sec:sample}) are answered independently by each community; within the present framework, they are the same questions.

\subsection{Approximation and separation from single-tower models}

We now study how well $\Mcal_d$ approximates a fixed target. Corollary~\ref{cor:trunc-lower} isolates the part of the error that no encoder can avoid. The remaining part is attributable to the encoders themselves. We assume that the encoder classes contain functions approaching the scaled modes $\sigma_k a_k$ and $b_k$ with errors $\eta_{f,k},\eta_{g,k}$ as in~\eqref{eq:realizability} (Assumption~\ref{asm:realizability}, Appendix~\ref{app:proofs-sec2}); for neural encoders, these errors follow standard ReLU approximation rates~\citep{yarotsky2017universal}.

\begin{theorem}[Approximation error decomposition]
\label{thm:err-split}
Let $F\in L^2(\mu_U\otimes\mu_V)$ have interaction spectrum $\{\sigma_k\}$. Under Assumption~\ref{asm:realizability}, with $B_f=\max_k\|\hat f_k\|_{L^2(\mu_U)}$ and $B_g=\max_k\|\hat g_k\|_{L^2(\mu_V)}$,
\begin{equation}
\inf_{\substack{f\in\mathcal{F}\\ g\in\mathcal{G}}}\|F-\inner{f}{g}\|_{L^2(\mu_U\otimes\mu_V)}^2
\;\le\;\underbrace{\sum_{k>d}\sigma_k^2}_{\text{truncation term}}
\;+\;\underbrace{C\sum_{k\le d}\bigl(\sigma_k^2\,\eta_{g,k}^2+B_g^2\,\eta_{f,k}^2\bigr)}_{\text{realization term}},
\end{equation}
where $C$ is an absolute constant.
\end{theorem}

\begin{proof}[Proof idea]
The error is split around the truncated Schmidt series, and the tensor-product product rule is applied term by term. Orthonormality of the ideal modes diagonalizes the two leading sums and avoids a factor of $d$; the remaining term is bounded by Cauchy-Schwarz. The full calculation, including the explicit constant, is deferred to Appendix~\ref{app:proofs-sec2}.
\end{proof}

The decomposition is tight up to the constant: the truncation term is intrinsic to the target, while the realization term depends on the model. Dense encoder classes therefore approach the lower bound of Corollary~\ref{cor:trunc-lower}. Whether training actually finds the coordinate alignment required by~\eqref{eq:realizability} is precisely the identifiability question addressed in Section~\ref{sec:gauge}.

Target smoothness controls the truncation term: an $s$-smooth target has spectrum $\sigma_k=O(k^{-s/m_V})$, while real-analytic targets decay exponentially. The precise tail bounds are given in Theorem~\ref{thm:decay} (Appendix~\ref{app:proofs-sec2}).

The dual-encoder head separates from single-tower models on both the approximation and the computational axes. On the approximation axis, it avoids the curse of the joint dimension: approximating a rank-$d$ target within $\varepsilon$ with ReLU encoders requires
\begin{equation}
N_{\mathrm{dual}}=O\bigl(d(\varepsilon^{-m_U/s}+\varepsilon^{-m_V/s})\bigr)
\end{equation}
parameters, whereas any single-tower model that treats $F$ as a general $(m_U+m_V)$-dimensional $H^s$ function requires $\Omega(\varepsilon^{-(m_U+m_V)/s})$ parameters in the worst case; hence $N_{\mathrm{dual}}/N_{\mathrm{single}}\to 0$ as $\varepsilon\to 0$ (Proposition~\ref{prop:separation}). On the compute axis, evaluating the head on all pairs costs $O((n+m)\,C_{\mathrm{net}}+nm\,d)$ against quadratic joint processing, recovering the complexity arguments of QK-attention matching~\citep{assignmentnet2025}, linear attention~\citep{katharopoulos2020linearattn}, and DeepONet~\citep{lu2021deeponet} (Proposition~\ref{prop:compute}, Appendix~\ref{app:proofs-sec2}).

\section{Identifiability as Gauge-Fixing}
\label{sec:gauge}

The theory of Section~\ref{sec:unif} characterizes what a rank-$d$ head can represent, but it does not specify which representation a training algorithm will actually find. The representation is far from unique: for any invertible matrix $A$, the pair $(Af,\,A^{-\top}g)$ represents exactly the same function. Consequently, the encoders of any trained model are defined only up to a symmetry of dimension $d^2$. This redundancy is not purely formal. It renders the population risk constant along large orbits, so the Hessian at a minimizer is highly degenerate and optimization becomes ill-posed. At the same time, it leaves the learned coordinates essentially arbitrary. Normalizations were introduced precisely to control these instabilities \citep{radford2021clip,triplebert2026}; a normalization makes the encoders meaningful exactly to the extent that it removes the symmetry.

We now formalize this view. Each normalization is a gauge-fixing choice, a rule for selecting one representative from each equivalence class. The residual symmetry that survives a constraint is precisely what remains arbitrary under it. The interaction spectrum constitutes the invariant content of the symmetry, and the gaps in that spectrum control how much identifiability any normalization can provide.

\begin{definition}[Gauge action]
\label{def:gauge}
For $A\in\GL_d$, the gauge transformation acts on the encoders by
\begin{equation}
\Phi_A(f,g)=\bigl(Af,\;A^{-\top}g\bigr),
\end{equation}
where $(Af)(u)=A\,f(u)$ and $\bigl(A^{-\top}g\bigr)(v)=A^{-\top}g(v)$. The set of all such transformations is the gauge group of the head.
\end{definition}

Gauge transformations leave the represented function, and therefore the population risk, unchanged:
\begin{equation}
\inner{Af(u)}{A^{-\top}g(v)}=\inner{f(u)}{g(v)}.
\end{equation}
The second moments $\Sigma_f=\E_u[ff^{\top}]$ and $\Sigma_g=\E_v[gg^{\top}]$ transform by congruence, so the eigenvalues of the product $\Sigma_f\Sigma_g$ are gauge invariants (Lemma~\ref{lem:gauge-inv}, Appendix~\ref{app:proofs-sec3}). Unfixed gauge also makes optimization ill-posed. Along the orbit of a minimizer the risk is flat, so the Hessian has at least $d^2-\dim S$ zero directions; without normalization this is generically $d^2$ (Proposition~\ref{prop:illposed}). A normalization pins exactly the subgroup whose residual freedom it eliminates; a smaller residual group therefore leaves fewer degenerate directions and yields a better-conditioned problem.

\paragraph{Normalizations as sections of the orbit.}
Two levels of normalization occur in practice, and only one affects identifiability. Embedding-level normalization constrains the encoders themselves and determines the residual gauge. Output-level normalization, such as the contrastive softmax in CLIP or top-$k$ selection in retrieval, acts on the score matrix after the inner product; it does not constrain the encoders and leaves the embedding-level gauge untouched. Table~\ref{tab:gauge} in Appendix~\ref{app:proofs-sec3} classifies the embedding-level schemes used in practice.

\paragraph{What each normalization identifies.}
The residual symmetry of a normalization indicates which transformations can act on a trained model without changing its output. The normalizations used in practice were adopted on heuristic grounds, for training stability or empirical feature decorrelation, rather than from an analysis of which symmetry they remove; the gauge view supplies the missing analysis and lets us rank them by residual symmetry. We characterize the two-sided schemes used in practice, in order of increasing identifiability, and find that only whitening removes the symmetry completely.

Cosine normalization is the weakest two-sided scheme \citep{radford2021clip}. Two-sided unit normalization $\tilde f=f/\|f\|$, $\tilde g=g/\|g\|$ with scores $\cos\angle(f(u),g(v))$, as employed by contrastive models such as CLIP~\citep{radford2021clip}, leaves a residual gauge exactly equal to the orthogonal group: for every rotation $R\in\Od$, the encodings $Rf$ and $Rg$ give pointwise identical cosine scores, and no other transformation does (Theorem~\ref{thm:cosine-od}). Consequently, the individual coordinates of a cosine-normalized embedding carry no intrinsic meaning; only rotation-invariant quantities (norms, pairwise inner products, and the eigenvalues of $\Sigma_f\Sigma_g$) are well-defined. This provides a theorem-level explanation of the widely reported uninterpretability of contrastive embedding dimensions, and it identifies the whitening scheme of Theorem~\ref{thm:whiten-id} as the remedy (Corollary~\ref{cor:cosine-uninterpretable}, Appendix~\ref{app:proofs-sec3}).

Two-sided nonnegativity, $f(u)\succeq0$ and $g(v)\succeq0$ (realized by a softplus or exponential output layer \citep{triplebert2026}), shrinks the residual group to the monomial group $\Perm\ltimes\Dplus$ of coordinate permutations composed with positive diagonal scalings. Under separability conditions~\citep{donoho2003nmf,arora2012nmf}, the encoders are identified up to permutation and scaling. The cost is an expressivity limitation: negatively correlated interaction modes are excluded (Theorem~\ref{thm:nonneg-nmf}, Appendix~\ref{app:proofs-sec3}).

Whitening is the unique scheme among those used in practice that removes the continuous gauge symmetry entirely. Cosine and nonnegativity each leave a nontrivial residual group and were adopted largely heuristically; by contrast, we prove that whitening admits a full identifiability guarantee: under a spectral gap it pins the interaction modes up to permutation and sign, recovering both the modes and the spectrum itself.

\begin{theorem}[Whitening fixes the gauge up to permutation and sign]
\label{thm:whiten-id}
Let the interaction singular values be distinct, $\sigma_1>\cdots>\sigma_d>\sigma_{d+1}$, and impose the whitening constraints
\begin{equation}
\Sigma_g=I_d,\qquad \Sigma_f=\Lambda=\diag(\lambda_1,\dots,\lambda_d),\quad \lambda_1\ge\cdots\ge\lambda_d\ge 0,
\label{eq:whiten}
\end{equation}
with $\Lambda$ diagonal and nonincreasing. Constraints of this form are standard in self-supervised representation learning, where they are used to decorrelate features~\citep{ermolov2021wmse,zbontar2021barlow,bardes2022vicreg}. Then every global minimizer of the population risk satisfies
\begin{equation}
f_k=\pm\,\sigma_k\,a_k,\qquad g_k=\pm\,b_k,\qquad \lambda_k=\sigma_k^2,\qquad k=1,\dots,d,
\label{eq:mode-recovery}
\end{equation}
with signs paired so that the product $f_k g_k=\sigma_k a_k b_k$ is uniquely determined, and with the coordinate labeling fixed by the ordering convention in~\eqref{eq:whiten}. The $k$-th coordinate of the whitened encoders recovers the $k$-th interaction mode of the target, and the diagonal of $\Lambda$ recovers the interaction spectrum. Equivalently, the encoders are identified up to permutation and sign.
\end{theorem}

\begin{proof}[Proof idea]
The rank-truncation bound of Corollary~\ref{cor:trunc-lower} implies that every minimizer attains the minimal risk. When $\sigma_d>\sigma_{d+1}$, this forces the represented function to equal the unique truncation $F_d$. The whitening constraints then require the coefficient matrices to satisfy $CC^{\top}=I$ and $CS^2C^{\top}=\Lambda$, where $S=\diag(\sigma_1,\dots,\sigma_d)$ has distinct eigenvalues; this forces $C$ to be a signed permutation. The full argument is in Appendix~\ref{app:proofs-sec3}.
\end{proof}

\paragraph{Implementing the whitening gauge.}
Hard constraints~\eqref{eq:whiten} can be enforced by minibatch whitening. Alternatively, a quadratic penalty inherits the identifiability of the hard constraint and improves conditioning monotonically with the penalty weight (Theorem~\ref{thm:soft-relax}, Appendix~\ref{app:proofs-sec3}).

In practice, the second moments are estimated from a finite sample. With probability $1-\delta$, the identified modes carry error of order $r/\Delta$, where $r=O\bigl(B^2\sqrt{\log(d/\delta)/n}\bigr)$, $B$ bounds the encoder outputs, and $\Delta=\min_k(\sigma_k^2-\sigma_{k+1}^2)$ is the spectral gap. Achieving accuracy $\varepsilon$ therefore requires
\begin{equation}
n\gtrsim\frac{B^4\log(d/\delta)}{\Delta^2\varepsilon^2}
\end{equation}
samples (Appendix~\ref{app:finite}). The same gap thus governs optimization conditioning, identifiability, and estimation simultaneously (Corollary~\ref{cor:gap-cond}).

Section~\ref{sec:exp} measures the gauge, conditioning, and expressivity consequences of the various schemes on controlled targets; the two real systems, branch-trunk operator learning~\citep{lu2021deeponet} and independently trained CLIP models, are treated in Appendix~\ref{app:exp} (Theorem~\ref{thm:cosine-od}). Appendix~\ref{app:proofs-sec3} lists the measurable quantities on which the schemes are compared.

\section{Sample Complexity and When Not to Use the Head}
\label{sec:sample}

The interaction spectrum controls both estimation error and fundamental limitations. The excess risk of a head fitted from \(n\) samples decomposes into an approximation term that decreases as the rank \(d\) grows, and an estimation term that increases with \(d\) through the complexity of the encoder classes. Balancing these two terms yields a rank-selection rule. When the spectrum is flat, however, every rank-\(d\) head suffers a relative-error floor that no amount of capacity or data can overcome; this yields a practical criterion for abandoning the architecture.

\paragraph{Setup and notation.}
We observe \(n\) independent samples \((u_i,v_i,y_i)\) with \(y_i=F(u_i,v_i)+\varepsilon_i\) and fit the head by empirical risk minimization over the class
\begin{equation}
\mathcal{H}_d=\bigl\{\,(u,v)\mapsto\inner{f(u)}{g(v)}:\ f\in\mathcal{F},\ g\in\mathcal{G}\,\bigr\}
\label{eq:headclass}
\end{equation}
of rank-\(d\) heads. The encoders are bounded, \(\|f(u)\|\le B_f\) and \(\|g(v)\|\le B_g\), so the scores are bounded by \(B=B_fB_g\). Let \(\mathfrak{R}_n\) denote the Rademacher complexity, extended to vector-valued classes coordinatewise. The class complexity collapses to the sum of the marginal complexities:
\begin{equation}
\mathfrak{R}_n(\mathcal{H}_d)\le 2d\bigl(B_g\,\mathfrak{R}_n(\mathcal{F})+B_f\,\mathfrak{R}_n(\mathcal{G})\bigr),
\end{equation}
for shared per-coordinate classes (Lemma~\ref{lem:rademacher}, Appendix~\ref{app:proofs-sec4}).

The excess risk of the fitted head therefore decomposes into the approximation terms of Theorem~\ref{thm:err-split} plus an estimation term controlled by this complexity:
\begin{equation}
\|\hat F-F\|_{L^2(\mu_U\otimes\mu_V)}^2
\;\lesssim\;
\sum_{k>d}\sigma_k^2
+\mathcal{E}_{\mathrm{real}}
+B\,d\bigl(\mathfrak{R}_n(\mathcal{F})+\mathfrak{R}_n(\mathcal{G})\bigr)
+B^2\sqrt{\frac{\log(1/\delta)}{n}},
\label{eq:excess}
\end{equation}
where \(\mathcal{E}_{\mathrm{real}}=C\sum_{k\le d}(\sigma_k^2\eta_{g,k}^2+B_g^2\eta_{f,k}^2)\) and \(\lesssim\) hides absolute constants, with probability \(1-\delta\) (Corollary~\ref{cor:excess}, Appendix~\ref{app:proofs-sec4}).

\paragraph{Linear encoders anchor the rate; rank selection.}
With linear encoders \(f(u)=U^{\top}u\), \(g(v)=V^{\top}v\) on inputs \(u\in\R^p\), \(v\in\R^q\), the head computes \(u^{\top}\Theta v\) with \(\Theta=UV^{\top}\) of rank at most \(d\). Estimation reduces to low-rank matrix sensing of the ground-truth matrix \(\Theta^{\star}\), whose minimax rate, with \(\sigma_{\varepsilon}^2\) the noise variance of \(\varepsilon_i\), is known exactly:
\begin{equation}
\E\|\hat\Theta-\Theta^{\star}\|_F^2\asymp\sigma_{\varepsilon}^2\,d(p+q)/n.
\end{equation}
This rate provides the baseline for the general bound (Theorem~\ref{thm:minimax}, Appendix~\ref{app:proofs-sec4}).

The rank \(d\) is the primary design choice for the head. It faces two opposing forces: the truncation term \(\sum_{k>d}\sigma_k^2\) decreases with \(d\), while the estimation term grows linearly with \(d\). The optimal rank balances these two contributions. When the spectrum decays exponentially, \(\sigma_k^2=O(e^{-2ck})\), balancing gives
\begin{equation}
d^{\star}\asymp\frac1{2c}\log\frac{n}{p+q},
\end{equation}
and yields a near-parametric rate
\begin{equation}
\|\hat F-F\|_{L^2}^2\lesssim\frac{(p+q)\log(n/(p+q))}{n}.
\end{equation}
Fast spectral decay thus eliminates the curse of dimensionality (Corollary~\ref{cor:opt-rank}, Appendix~\ref{app:proofs-sec4}). The three terms in~\eqref{eq:excess} form the complete error ledger: fast decay keeps \(d^{\star}\) small; slow decay makes the head expensive in every respect. The rest of this section analyzes the regime in which even the optimal \(d^{\star}\) is inadequate.

\paragraph{Identifiability has a sample cost.}
Identifiability, as characterized in Section~\ref{sec:gauge}, is a population-level statement. Realizing it from finite samples incurs an additional cost, detailed in Appendix~\ref{app:finite} (Corollary~\ref{cor:whiten-sc}). The total mode error of a fitted and subsequently whitened head (Proposition~\ref{prop:risk-mode}) decomposes into the excess risk of Corollary~\ref{cor:excess} plus the empirical whitening error. Both terms are divided by the spectral gap \(\Delta\) and decay as \(n^{-1/2}\).

\paragraph{Late versus early interaction.}
We now compare the multiplicative head, which we refer to as a late-interaction model, with early-interaction models in which the two inputs are combined during encoding. By Corollary~\ref{cor:trunc-lower}, every rank-\(d\) head pays at least the spectral tail, so the spectrum determines the outcome.

Define the \emph{effective interaction rank}
\begin{equation}
d_{\varepsilon}(F)=\min\Bigl\{d:\frac{\sum_{k>d}\sigma_k^2}{\sum_{k\ge1}\sigma_k^2}\le\varepsilon\Bigr\}
\end{equation}
as the smallest embedding dimension for which a rank-\(d\) head achieves relative error \(\varepsilon\). The spectrum is called \emph{flat} at scale \(N\) when \(d_{\varepsilon}(F)=\Omega((1-\varepsilon)N)\).

The equality function realizes this worst case exactly. On \(\mathcal{U}=\mathcal{V}=\{\pm1\}^m\) with \(N=2^m\), the interaction operator satisfies \(T_{F_{\mathrm{eq}}}=\frac1N\mathrm{Id}\). The singular values are all equal to \(1/\sqrt{N}\), the interaction rank is \(N\), and every rank-\(d\) head achieves relative error exactly \(1-d/N\). Thus, attaining relative error \(\varepsilon\) requires \(d\ge(1-\varepsilon)2^m\) embedding dimensions, exponential in the input size (Theorem~\ref{thm:flat-floor}, Appendix~\ref{app:proofs-sec4}).

By contrast, an early-interaction predictor with \(O(m)\) parameters represents \(F_{\mathrm{eq}}\) exactly. The identity \(\inner{u}{v}=m-2\,\mathrm{Ham}(u,v)\) allows a simple threshold on the inner product to recover the equality indicator, whether implemented by a cross-attention layer with identity projections or by a joint MLP with \(O(m)\) ReLUs. Combined with Theorem~\ref{thm:flat-floor}, this yields an exponential separation: \(O(m)\) parameters suffice under early interaction, while any late-interaction head requires \(\Omega(2^m)\) dimensions to reach constant relative error (Theorem~\ref{thm:early-escape}, Appendix~\ref{app:proofs-sec4}).

These observations lead to a concrete usability criterion (Theorem~\ref{thm:criterion}, Appendix~\ref{app:proofs-sec4}). When the effective interaction rank \(d_{\varepsilon}(F)\) is small, a late-interaction head is appropriate: its approximation and estimation errors are controllable, and its computational and statistical advantages are preserved. When the spectrum is flat, the truncation floor forces every rank-\(d\) head to relative error at least \(\varepsilon\), while early-interaction models can still reach \(\varepsilon\) with polynomial resources when the target has low-dimensional joint structure. The deciding quantity, the spectral decay rate, can be estimated from data by a weighted singular value decomposition of the empirical kernel matrix, making the criterion actionable before either model is trained.

The criterion is a boundary, not a verdict. The early-interaction upper bounds assume low-dimensional joint structure; a target that is both high-rank and unstructured is difficult for both architectures. Moreover, the parameter counts in Theorem~\ref{thm:early-escape} are statements about representational existence, not about learnability. Section~\ref{sec:exp} instantiates both sides of the dichotomy: on gapped spectra the head identifies and estimates well, while flat and narrow-band targets drive it to the exponential-rank floor.
\section{Experiments}
\label{sec:exp}

Our experiments are designed for mechanism validation rather than benchmark chasing: each study isolates a quantity that the theory predicts exactly. The controlled synthetic setting (Section~\ref{sec:exp-synth}) supplies exact ground truth and is reported in full, with the tables and figures supporting each conclusion placed alongside the corresponding paragraph. Two real architectures instantiate the multiplicative head directly, DeepONet and CLIP, and are reported in full in Appendix~\ref{app:exp} due to page limitation; here we state only the conclusions they establish. On DeepONet, post-hoc whitening recovers the analytic operator eigenbasis from an end-to-end-trained model. On two pretrained CLIP pairs, the interaction spectra are shared across models while per-dimension concept probes are not, and fitting the single residual rotation reconciles them.


\subsection{Controlled synthetic study}
\label{sec:exp-synth}

\paragraph{Setup.}
This study validates the identifiability and usability claims of Sections~\ref{sec:gauge} and~\ref{sec:sample} in a setting with exact ground truth. We fix a target \(F(u,v)=\sum_k\sigma_k a_k(u)b_k(v)\) with shifted Legendre modes and a prescribed spectrum, fit linear dual encoders over a fixed Legendre feature map, and record four quantities per run: the fit error, the alignment of the recovered modes with the true \(a_k,b_k\), the encoder-covariance conditioning, and the cross-seed gauge distance. The encoder-realization error of Assumption~\ref{asm:realizability} is zero, so normalization is the only variable across runs. The four paragraphs below use these quantities to test, in turn, the identifiability ranking of the gauges (Theorem~\ref{thm:whiten-id}), the gap law of estimation (Corollaries~\ref{cor:gap-cond} and~\ref{cor:whiten-sc}), the sample-complexity rate (Theorem~\ref{thm:minimax}), and the flat-spectrum floor (Theorem~\ref{thm:flat-floor}).

\paragraph{Normalization as gauge-fixing.}
The seven schemes of Table~\ref{tab:zoo} separate cleanly by residual group. Only whitening attains perfect mode alignment (\(1.000\)) with zero cross-seed gauge distance and drives the conditioning diagnostic \(\max(\mathrm{cond}\,\Sigma_f,\mathrm{cond}\,\Sigma_g)\) to its theoretical floor \(\sigma_1^2/\sigma_d^2\approx 9.5\) (Corollary~\ref{cor:gap-cond}); every looser gauge remains stuck at alignment \(0.53\) to \(0.65\), its residual group leaving the modes rotationally entangled (Theorem~\ref{thm:cosine-od}). The gauge-invariant diagnostic \(\mathrm{cond}(\Sigma_f\Sigma_g)\) cannot separate the schemes (Lemma~\ref{lem:gauge-inv}). The two-sided nonnegative gauge incurs error \(4.25\), the predicted out-of-class cost of excluding negatively correlated modes (Theorem~\ref{thm:nonneg-nmf}). Normalization is therefore gauge fixing, and whitening is the only scheme that removes the symmetry completely, exactly as Theorem~\ref{thm:whiten-id} predicts.

\begin{table}[t]
\centering
\small
\setlength{\tabcolsep}{3pt}
\caption{The seven normalization schemes, four seeds each. Residual gauge: the subgroup that survives the constraint (Theorem~\ref{thm:whiten-id} predicts $\Perm\ltimes\Dpm$ for whitening, $\Od$ for cosine, $\Perm\ltimes\Dplus$ for two-sided nonnegativity). Fit MSE: squared regression error. Alignment: mean mode alignment of the two encoders after optimal permutation and sign matching. Cond: $\max(\mathrm{cond}\,\Sigma_f,\mathrm{cond}\,\Sigma_g)$, gauge dependent. Cross-seed: mean pairwise gauge distance between seeds.}
\label{tab:zoo}
\begin{tabular}{@{}>{\raggedright\arraybackslash}p{2.4cm}>{\raggedright\arraybackslash}p{2.3cm}r r r r@{}}
\toprule
Normalization & Residual gauge & Fit MSE & Align. ($f$,$g$) & Cond & Cross-seed \\
\midrule
none & $\GL_d$ & 0.0004 & 0.605, 0.619 & 13.8 & 0.326 \\
l2\_single & $\approx\GL_d$ & 0.0092 & 0.561, 0.420 & 16.8 & 0.284 \\
per\_coord & $\Od$+shear & 0.0004 & 0.607, 0.593 & 18.7 & 0.330 \\
clip\_cosine & $\Od$ & 0.0568 & 0.529, 0.550 & 305.9 & 0.212 \\
nonneg & $\Perm\ltimes\Dplus$ & 4.251 & 0.576, 0.446 & 121.0 & 0.280 \\
whiten\_soft & $\to\Perm\ltimes\Dpm$ & 0.0169 & 0.651, 0.549 & 105.5 & 0.348 \\
whiten\_hard & $\Perm\ltimes\Dpm$ & 0.0004 & 1.000, 1.000 & 9.5 & 0.000 \\
\bottomrule
\end{tabular}
\end{table}

\paragraph{The gap as the single unifying quantity.}
The spectral gap is the single quantity governing identification. With the finite-sample whitening estimator of Appendix~\ref{app:finite}, we sweep the gap at fixed sample size and the sample size at fixed gap, eight configurations in all (Table~\ref{tab:gap}; Figure~\ref{fig:gap} in Appendix~\ref{app:exp} plots the collapse). The mode-identification error scales as \(1/\Delta\) at fixed sample size and as \(n^{-1/2}\) at fixed gap, where \(\Delta=\min_k(\sigma_k^2-\sigma_{k+1}^2)\), and the two sweeps collapse onto the single relation
\begin{equation}
\mathrm{err}\cdot\Delta\cdot\sqrt{n}\approx 2.48
\end{equation}
with coefficient of variation \(0.12\): gap and sample size enter the identification error only through their product, the concrete face of Corollaries~\ref{cor:gap-cond} and~\ref{cor:whiten-sc}.

\begin{table}[t]
\centering
\small
\setlength{\tabcolsep}{4pt}
\caption{Gap and sample-size sweeps of the finite-sample whitening estimator (Appendix~\ref{app:finite}). Left group: the spectral gap $\Delta=\min_k(\sigma_k^2-\sigma_{k+1}^2)$ swept by scaling a base spectrum at fixed $n=3200$. Right group: sample size swept at fixed $\Delta=0.900$. Mode error is the mean mode identification error; the product $\mathrm{err}\cdot\Delta\cdot\sqrt{n}$ is near-constant across both groups, $\approx 2.48$.}
\label{tab:gap}
\begin{tabular}{@{}r r r | r r@{}}
\toprule
\multicolumn{3}{c}{Fixed $n=3200$} & \multicolumn{2}{c}{Fixed $\Delta=0.900$} \\
Scale & $\Delta$ & Mode err & $n$ & Mode err \\
\midrule
1.00 & 0.900 & 0.0431 & 200 & 0.1979 \\
0.60 & 0.648 & 0.0782 & 800 & 0.0956 \\
0.35 & 0.417 & 0.1258 & 3200 & 0.0470 \\
0.20 & 0.230 & 0.2458 & 12800 & 0.0236 \\
\bottomrule
\end{tabular}
\end{table}

\paragraph{The sample-complexity rate law.}
In the linear regime the head is exactly a rank-\(d\) matrix sensing problem, and Theorem~\ref{thm:minimax} predicts \(\E\|\hat\Theta-\Theta^{\star}\|_F^2\asymp\sigma_{\varepsilon}^2 d(p+q)/n\). We sweep each of the three variables with the others held fixed and fit log-log slopes, the three sweeps shown in Figure~\ref{fig:sampcomp}. The measured slopes are \(-1.03\) in \(n\), \(+0.89\) in \(d\), and \(+1.24\) in \(p+q\), against the theoretical \(-1,+1,+1\); the excess in the last slope is the expected Marchenko-Pastur finite-sample correction, and gradient descent on the factored pair matches the least-squares slope. These sweeps reject the loose general bound of Remark~\ref{rem:conservative}, which would predict slope \(-1/2\): estimation is governed by the sum of the marginal encoder complexities, and the \(1/n\) dependence is genuine.

\begin{figure}[t]
\centering
\includegraphics[width=\textwidth]{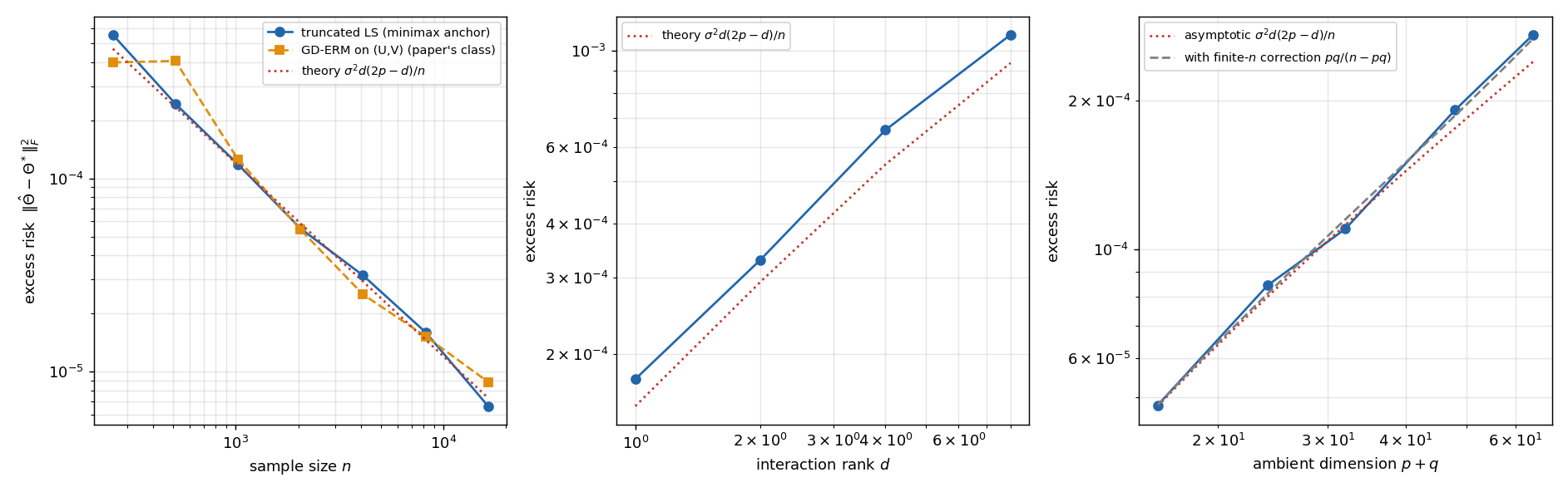}
\caption{The sample-complexity law of Theorem~\ref{thm:minimax}. Log-log sweeps of the Frobenius excess risk against sample size $n$ (left), interaction rank $d$ (center), and ambient dimension $p+q$ (right). Points are the measured risk; the dotted line is the theoretical rate $\sigma^2 d(2p-d)/n$, with the finite-$n$ corrected form (dashed) in the right panel.}
\label{fig:sampcomp}
\end{figure}

\paragraph{The usability criterion.}
A flat spectrum floors the head regardless of capacity. We test two flat and near-flat regimes, both tabulated in Tables~\ref{tab:eq} and~\ref{tab:band}. On the Boolean equality kernel \(F=\mathbf{1}[u=v]\) on \(\{\pm1\}^m\), which has a completely flat spectrum, a rank-\(d\) head incurs relative error exactly \(1-d/N\) and requires \(d\approx 0.9\cdot 2^m\) dimensions to reach \(10\%\) error, while an early-interaction threshold on \(\inner{u}{v}\) is exact with \(2m+1\) parameters (Theorems~\ref{thm:flat-floor} and~\ref{thm:early-escape}). Smooth narrow-band kernels give the polynomial counterpart \(d_{\varepsilon}=\Theta(1/h)\) (Remark~\ref{rem:smooth}). A trained-model anchor on a block-boxcar kernel realizes the same floor (Table~\ref{tab:boxcar} and Figure~\ref{fig:usability}, Appendix~\ref{app:exp}).

\begin{table}[t]
\centering
\small
\begin{minipage}[t]{0.6\textwidth}
\centering
\setlength{\tabcolsep}{4pt}
\caption{The boolean equality kernel $F=\mathbf{1}[u=v]$ on $\{\pm1\}^m$, $N=2^m$. Rank $d$: embedding dimension at which the head first reaches $10\%$ relative error. Early: parameters $2m+1$ and error of the early-interaction comparator.}
\label{tab:eq}
\begin{tabular}{@{}r r r r r r@{}}
\toprule
$m$ & $N$ & Rank $d$ & $d/N$ & Early params & Early error \\
\midrule
3 & 8 & 8 & 1.00 & 7 & 0 \\
4 & 16 & 15 & 0.94 & 9 & 0 \\
5 & 32 & 29 & 0.91 & 11 & 0 \\
6 & 64 & 58 & 0.91 & 13 & 0 \\
7 & 128 & 116 & 0.91 & 15 & 0 \\
8 & 256 & 231 & 0.90 & 17 & 0 \\
\bottomrule
\end{tabular}
\end{minipage}\hfill
\begin{minipage}[t]{0.35\textwidth}
\centering
\setlength{\tabcolsep}{4pt}
\caption{Narrowband periodic-Gaussian kernel of bandwidth $h$ on the circle: interaction rank $d_{\varepsilon}$ for $\varepsilon=10\%$ error. The rank grows as $\Theta(1/h)$ (Remark~\ref{rem:smooth}).}
\label{tab:band}
\begin{tabular}{@{}r r r@{}}
\toprule
$h$ & $1/h$ & $d_{\varepsilon}$ \\
\midrule
0.200 & 5 & 3 \\
0.100 & 10 & 4 \\
0.050 & 20 & 8 \\
0.025 & 40 & 15 \\
0.0125 & 80 & 30 \\
\bottomrule
\end{tabular}
\end{minipage}
\end{table}

\section{Conclusion}
\label{sec:concl}

This paper has examined a single architectural pattern, the multiplicative dual-encoder head, through a unifying lens: the interaction spectrum of the target. Through theoretical analysis and numerical experiments, we answer four central questions.  (i) Approximation error decomposes into a truncation term intrinsic to the target and a realization term intrinsic to the encoders, with target smoothness controlling the decay rate. (ii) Sample complexity is governed by the sum of the two marginal encoder complexities, with the linear regime anchoring the minimax rate. (iii) Usability is determined by the decay rate itself: a flat spectrum imposes an error floor that no rank-$d$ head can escape, while early-interaction models can, and the deciding quantity is measurable from data before either model is trained. (iv) Identifiability, the problem that the framework brings into focus, is gauge fixing: every normalization used in practice is a section of the gauge orbit; whitening pins the interaction modes up to permutation and sign; and the spectral gap simultaneously controls optimization conditioning, identifiability, and estimation. 

\section*{Ethics Statement}
\emph{This work adheres to the principles outlined in the ICLR Code of Ethics.}

\bibliography{refs}
\bibliographystyle{iclr2026_conference}

\appendix
\newpage
\section*{Appendix Contents}  
\startcontents
\printcontents{}{1}{\setcounter{tocdepth}{2}}
\newpage

\section{Derivations for the Unification Table}
\label{app:unif}

This appendix spells out, for each of the ten families of Table~\ref{tab:unif}, the identification that makes it an instance of the multiplicative dual-encoder head \eqref{eq:head}: the two input spaces, the encoder maps, the score, the reference measures that define the interaction spectrum, and the boundary cases where an additional output-level normalization acts after the inner product and therefore leaves the gauge untouched. Table~\ref{tab:unif} collects the ten families, each computing a real-valued output for a pair of inputs as the inner product of two separately learned encodings; reading across the rows makes visible what is shared, and the paragraphs below identify the interaction operator and its spectrum for each family, including the boundary cases.

\paragraph{CLIP.} Inputs are an image $I$ and a text $T$; the encoders are the image tower $f$ and the text tower $g$, both $\ell_2$-normalized; the score is $\mathrm{sim}(I,T)=\inner{f(I)}{g(T)}$, rescaled by a temperature and fed into a softmax over the batch. The reference measures are the empirical distributions over images and over texts in the training batch. The temperature and the softmax act on the similarity matrix after the inner product, so they are output-level and do not affect the embedding-level gauge; the unit-norm constraints are embedding-level and leave exactly the residual rotation $\Od$ of Theorem~\ref{thm:cosine-od}.

\paragraph{DeepONet.} Inputs are an input function $a$ (on a sensor grid) and a query point $y$; the encoders are the branch net $b$ and the trunk net $t$; the score is $\mathcal{G}(a)(y)\approx\inner{b(a)}{t(y)}$ \citep{lu2021deeponet}. The reference measures are the distribution over input functions and the measure over query points. For a linear operator with kernel $\kappa=\sum_k\sigma_k\beta_k(y)\gamma_k(x)$ and a white input measure, the interaction spectrum coincides with the operator singular values and the trunk's ideal basis with the left singular functions, which is the ground truth used in Appendix~\ref{app:exp-deeponet}. The boundary case is the scale imbalance between branch and trunk outputs, familiar in practice; it is the scale redundancy of the gauge group, diagnosed in Section~\ref{sec:gauge}. POD-DeepONet substitutes a precomputed orthogonal POD basis for the learned trunk~\citep{lu2021deeponet}, motivated by the same nonuniqueness; our treatment is complementary, characterizing when the learned basis becomes identifiable instead of replacing it (Appendix~\ref{app:exp-deeponet}).

\paragraph{Two-tower retrieval.} Inputs are a query $q$ and a document $d$; the encoders are two towers $\varphi,\psi$; the score is a relevance measure $\mathrm{rel}(q,d)=\inner{\varphi(q)}{\psi(d)}$ \citep{karpukhin2020dpr}. The reference measures are the query distribution and the document distribution, and the interaction spectrum is the relevance spectrum of the underlying ranker. The top-$k$ nearest-neighbor search is output-level.

\paragraph{UVFA (universal value function approximators).} Inputs are a state $s$ and a goal $g$; the encoders are the state encoder $\phi(s)$, a learned map from states to features, and the goal encoder $\psi(g)$, a learned map from goals to features; the score is the goal-conditioned value $V(s,g)=\inner{\phi(s)}{\psi(g)}$ \citep{schaul2015uvfa}. The reference measures are the state-visitation distribution and the goal distribution. Goal-conditioned value functions are naturally low-rank when goals and states share a low-dimensional feature structure, which is exactly the case in which the representation \eqref{eq:head} is sample-efficient (Section~\ref{sec:sample}). The interaction spectrum also gives these methods a rank-selection rule they currently lack: the number of features in a successor-feature representation or the width of a value-network embedding is typically a tuned hyperparameter, whereas Corollary~\ref{cor:opt-rank} ties it to the decay of the spectrum.

\paragraph{BVN (bilinear value networks).} Inputs are a state-action pair $(s,a)$ and a state-goal pair $(s,g)$; the encoders are the MLP $f(s,a)$ encoding the state-action pair and the MLP $\varphi(s,g)$ encoding the state-goal pair; the score is $Q(s,a,g)=\inner{f(s,a)}{\varphi(s,g)}$ \citep{hong2022bvn}. The reference measures are the joint distributions over $(s,a)$ and $(s,g)$. Note that the two arguments share the state $s$: the head's theory applies to the function $Q$ of the pair of arguments, and the shared state is exactly the coupling that makes $Q$ non-separable in $(s,a,g)$ directly, which is why the two-input factorization matters.

\paragraph{Successor features.} Inputs are a state $s$ and a reward weight vector $w$; the encoder on the state side is the successor-feature map $\phi(s)$ and the encoder on the weight side is the identity $w$ (a singleton encoder class, not learned); the score is $V_w(s)=\inner{\phi(s)}{w}$ \citep{barreto2017sf}. This is the head with $\mathcal{G}=\{\mathrm{id}\}$, whose Rademacher complexity is zero, so the sample-complexity bounds of Section~\ref{sec:sample} specialize to the standard guarantee for successor features: only the state-side class matters.

\paragraph{DistMult and ComplEx.} Inputs are a head entity $h$ and a tail entity $t$; the encoders are entity embeddings $e_h,e_t\in\R^d$; the score is the bilinear form $\inner{e_h}{r\circ e_t}$ of DistMult or the real part of the complex product $e_h^{\top}\,\mathrm{diag}(r)\,\overline{e_t}$ of ComplEx, where $r$ is the relation embedding \citep{yang2015distmult,trouillon2016complex}. Both are rank-$d$ bilinear forms on the entity embeddings, i.e., heads with linear encoders and a fixed middle matrix; the interaction spectrum is governed by the Gram structure of the entity embeddings under the relation. The reference measure is the empirical distribution over observed triples.

\paragraph{Linear attention.} Inputs are a query $q$ and a key $k$ in a common space; the encoders are the feature map $\phi$ applied to each; the attention weight is $\inner{\phi(q)}{\phi(k)}$ \citep{katharopoulos2020linearattn}. The aggregation step $\sum_j\phi(k_j)v_j^{\top}$ makes the compute separation of Proposition~\ref{prop:compute} exact: the pairwise inner products are never formed, giving the $O(n^2)\to O(n)$ complexity. The reference measure is the distribution over keys (or key-value pairs). Finite-dimensional feature maps realize the head exactly; softmax attention with random-feature maps is the limiting case of an infinite-dimensional $\phi$.

\paragraph{Factorization machines.} Inputs are a feature index $i$ and a feature index $j$; the encoders are the embedding maps $v_i,v_j$; the pairwise term is $x_ix_j\inner{v_i}{v_j}$ \citep{rendle2010fm}. The reference measure is the empirical distribution over feature pairs weighted by the data. The interaction spectrum of the pairwise term is the spectrum of the Gram matrix of the embeddings, and the rank selection of Section~\ref{sec:sample} is the standard FM capacity control viewed through the interaction spectrum.

\paragraph{QK-attention matching.} Inputs are a worker $w$ and an order $o$; the encoders are a BERT-style encoder $f$ and an MLP $g$; the score is the match value $\inner{f(w)}{g(o)}$ forming the $Q$-matrix of a bipartite assignment problem, with a one-sided softplus and unit-norm constraint on the order encoder \citep{triplebert2026,assignmentnet2025,zhao2025impacts}. The reference measure is the empirical distribution over worker-order pairs. The matching layer that consumes the score matrix is output-level; the one-sided constraint is embedding-level but, as Remark~\ref{rem:gauge-scale} explains, it is a stability fix rather than an identifiability scheme. The two papers differ only in the downstream use of the same head, which is why they share the identifiability properties derived in this paper.

\begin{table}[t]
\centering
\small
\setlength{\tabcolsep}{3pt}
\caption{Ten families of methods as instances of the multiplicative dual-encoder head \eqref{eq:head}: for each family, the two inputs, the encoder maps, the output, and the downstream use.}
\label{tab:unif}
\begin{tabular}{@{}>{\raggedright\arraybackslash}p{1.9cm}>{\raggedright\arraybackslash}p{2.0cm}>{\raggedright\arraybackslash}p{2.7cm}>{\raggedright\arraybackslash}p{3.0cm}>{\raggedright\arraybackslash}p{3.3cm}@{}}
\toprule
Instance & Inputs & Encoders & Output & Downstream use \\
\midrule
CLIP \citep{radford2021clip} & image, text & ViT, text transformer & normalized similarity & contrastive softmax over a batch \\
DeepONet \citep{lu2021deeponet} & function $a$, point $y$ & branch net, trunk net & $\inner{\mathrm{branch}(a)}{\mathrm{trunk}(y)}$ & operator output $\mathcal{G}(a)(y)$ \\
Two-tower retrieval \citep{karpukhin2020dpr} & query, document & two towers & relevance & top-$k$ ANN search \\
UVFA \citep{schaul2015uvfa} & state $s$, goal $g$ & $\phi(s)$, $\psi(g)$ & $V(s,g)$ & value function \\
BVN \citep{hong2022bvn} & $(s,a)$, $(s,g)$ & two MLPs & $Q(s,a,g)$ & multi-goal RL \\
Successor features \citep{barreto2017sf} & state $s$, weights $w$ & $\phi(s)$, $w$ & $V_w(s)$ & transfer \\
DistMult / ComplEx \citep{yang2015distmult,trouillon2016complex} & head, tail entity & entity embeddings & bilinear / complex product & link prediction \\
Linear attention \citep{katharopoulos2020linearattn} & query $q$, key $k$ & feature map $\phi$ & $\inner{\phi(q)}{\phi(k)}$ & $O(n^2)\to O(n)$ aggregation \\
Factorization machines \citep{rendle2010fm} & feature $i$, feature $j$ & embedding $v_i$ & $\inner{v_i}{v_j}x_ix_j$ & pairwise interaction \\
QK-attention matching \citep{triplebert2026,assignmentnet2025} & worker $w$, order $o$ & BERT, MLP & $\inner{f(w)}{g(o)}$ & $Q$-matrix for ILP matching \\
\bottomrule
\end{tabular}
\end{table}

\paragraph{Positioning against prior work.}
The paragraphs above identify each family and its interaction operator; we now position the paper's claims against the closest prior work, none of which treats the families as instances of a shared architecture or formulates the identifiability question. On the estimation side, the linear case of the head is rank-$d$ matrix sensing, and the paper relies on the classical minimax theory, upper bounds via restricted-isometry-type analyses~\citep{candes2011tight,negahban2011estimation} and matching lower bounds~\citep{rohde2011estimation}; bilinear bandits~\citep{jun2019bilinear} study the online counterpart. The contribution relative to this line is a translation rather than a new bound: the estimation problem of a trained head is identified with matrix sensing, and the minimax rate is connected to the rank-selection rule of Section~\ref{sec:sample} through the decay of the interaction spectrum, which the matrix-sensing literature does not formulate because it treats the rank as given. On the identifiability side, the direct precursor of our remedy is self-supervised whitening: W-MSE applies hard whitening to the embeddings~\citep{ermolov2021wmse}, and Barlow Twins and VICReg implement its soft relaxations, pushing the cross-correlation matrix toward the identity and regularizing variance and covariance~\citep{zbontar2021barlow,bardes2022vicreg}; it is also known that whitening reduces a general linear gauge to an orthogonal ambiguity, a polar-decomposition observation. Against this background, our contribution is not the whitening operation itself, which we do not claim, but three things: the interpretation of normalization as a section of the gauge orbit and the classification of the residual subgroups across the normalizations used in practice (Table~\ref{tab:gauge}, Appendix~\ref{app:proofs-sec3}); the theorem that whitening together with a spectral gap pins the modes down to permutation and sign, strictly beyond the orthogonal ambiguity that decorrelation alone leaves (Theorem~\ref{thm:whiten-id}); and the statement that the interaction spectrum is the unique gauge-invariant content of a trained model, an interpretability claim with a theorem behind it rather than an empirical observation. The nonnegative variant of the bilinear form has its own classical identifiability theory, separability-type conditions for uniqueness up to permutation and scale~\citep{donoho2003nmf,arora2012nmf}, which the gauge analysis of Section~\ref{sec:gauge} subsumes as a special case (Theorem~\ref{thm:nonneg-nmf}).

\section{Proofs for Section 2}
\label{app:proofs-sec2}

We prove the results of Section~\ref{sec:unif} in order. Throughout, $\|\cdot\|$ denotes the $L^2(\mu_U\otimes\mu_V)$ norm, and $\|\cdot\|_{L^2(\mu_U)}$, $\|\cdot\|_{L^2(\mu_V)}$ the marginal norms.

\begin{assumption}[Encoder realizability]
\label{asm:realizability}
The encoder classes $\mathcal{F}$ and $\mathcal{G}$ contain functions $\hat f,\hat g$ such that, for some nonnegative errors $\eta_{f,k},\eta_{g,k}$,
\begin{equation}
\|\hat f_k-\sigma_k a_k\|_{L^2(\mu_U)}\le\eta_{f,k},\qquad
\|\hat g_k-b_k\|_{L^2(\mu_V)}\le\eta_{g,k},\qquad k=1,\dots,d.
\label{eq:realizability}
\end{equation}
\end{assumption}

\begin{remark}[What the decomposition buys]
The truncation term is intrinsic to the problem, the realization term to the model; when the encoder classes are dense in $L^2$ the bound approaches the lower bound of Corollary~\ref{cor:trunc-lower}, so the decomposition is tight up to the constant. Whether training finds the coordinate alignment of \eqref{eq:realizability} is the question of identifiability addressed in Section~\ref{sec:gauge}.
\end{remark}

\begin{proposition}[Compute separation]
\label{prop:compute}
Evaluating a rank-$d$ head on all $n\times m$ input pairs costs
\begin{equation}
O\bigl((n+m)\,C_{\mathrm{net}}+nm\,d\bigr),
\end{equation}
where $C_{\mathrm{net}}$ is the cost of one encoder forward pass, whereas a model that processes each pair jointly costs $O(nm\,C_{\mathrm{net}})$. When the encoder forward pass dominates the inner product, the head is cheaper by a factor $C_{\mathrm{net}}/d$.
\end{proposition}

\begin{remark}[Known complexity results as instances]
Proposition~\ref{prop:compute} recovers, as special cases, the multiplicative-to-additive complexity argument of the QK-attention matching networks \citep{triplebert2026}, the $O(n^2)\to O(n)$ aggregation of linear attention \citep{katharopoulos2020linearattn}, where the attention weight is written as the inner product of feature maps, and the efficient branch-trunk evaluation of DeepONet \citep{lu2021deeponet}.
\end{remark}

\begin{remark}[Where low interaction rank comes from]
Three structural sources make the truncation term small: exact low rank ($F=\sum_{j\le r}\psi_j(u)\chi_j(v)$ implies $\irank(F)\le r$), tensor-product structure (for $F$ in a tensor-product RKHS the spectrum is governed by the marginal kernels), and smoothness (Theorem~\ref{thm:decay} gives explicit rates). Conversely, a flat spectrum, developed in Section~\ref{sec:sample}, means the truncation term does not decay and no choice of $d$ helps.
\end{remark}

\begin{lemma}[Rank equivalence]
\label{lem:rank-equiv}
A function $S\in L^2(\mu_U\otimes\mu_V)$ can be written as $S(u,v)=\inner{f(u)}{g(v)}$ for some encoders $f:\mathcal{U}\to\R^d$ and $g:\mathcal{V}\to\R^d$ if and only if $\rank(T_S)\le d$. In particular, the parametric class of eq.~\eqref{eq:class} coincides with $\Mcal_d=\{F:\irank(F)\le d\}$, and we use the two descriptions interchangeably.
\end{lemma}

\paragraph{Proof of Lemma~\ref{lem:rank-equiv} (Rank equivalence).}
The forward direction is immediate: if $S=\sum_{k\le d}f_k\otimes g_k$, then the interaction operator is a sum of at most $d$ rank-one operators,
\begin{equation}
T_S=\sum_{k\le d}\inner{\cdot}{g_k}_{L^2(\mu_V)}\,f_k,
\end{equation}
so $\rank(T_S)\le d$. Conversely, if $\rank(T_S)=r\le d$, the Schmidt decomposition of $S$ reads $S=\sum_{k\le r}\sigma_k\,a_k\otimes b_k$ with $\sigma_k>0$, and the encoders $f=(f_1,\dots,f_d)$, $g=(g_1,\dots,g_d)$ defined by $f_k=\sigma_k a_k$ and $g_k=b_k$ for $k\le r$, and $f_k=g_k=0$ for $k>r$, satisfy $S=\inner{f}{g}$.

\begin{corollary}[Spectral truncation lower bound]
\label{cor:trunc-lower}
For every rank-$d$ head $(f,g)$,
\begin{equation}
\|F-\inner{f}{g}\|_{L^2(\mu_U\otimes\mu_V)}^2\;\ge\;\sum_{k>d}\sigma_k^2,
\end{equation}
and the bound is attained in $L^2$ by the truncated Schmidt series $F_d=\sum_{k\le d}\sigma_k\,a_k\otimes b_k$, which is the unique best rank-$d$ approximation whenever $\sigma_d>\sigma_{d+1}$.
\end{corollary}

\paragraph{Proof of Corollary~\ref{cor:trunc-lower} (Spectral truncation lower bound).}
The map $S\mapsto T_S$ is an isometric isomorphism between $L^2(\mu_U\otimes\mu_V)$ and the Hilbert-Schmidt class $\mathrm{HS}(L^2(\mu_V),L^2(\mu_U))$ equipped with the inner product $\inner{T_S}{T_{S'}}=\tr(T_S^*T_{S'})$; indeed $\|T_S\|_{\mathrm{HS}}^2=\sum_{k\ge1}\sigma_k^2=\|S\|_{L^2}^2$. Under this isometry, the rank constraint of Lemma~\ref{lem:rank-equiv} becomes the operator rank constraint, and the Eckart-Young-Mirsky theorem for compact operators (see e.g. Gohberg and Krein) gives that the best rank-$d$ approximation of $T_F$ is the truncation to its leading $d$ singular triples, with squared error $\sum_{k>d}\sigma_k^2$. Uniqueness: when $\sigma_d>\sigma_{d+1}$, the leading $d$-dimensional singular subspace of $T_F$ is unique, hence $F_d=\sum_{k\le d}\sigma_k a_k\otimes b_k$ is the unique minimizer.

\paragraph{Proof of Theorem~\ref{thm:err-split} (Approximation error decomposition).}
Fix realizations $\hat f\in\mathcal{F}$, $\hat g\in\mathcal{G}$ satisfying Assumption~\ref{asm:realizability}, and write $f_k^\star=\sigma_k a_k$, $g_k^\star=b_k$ for the ideal modes, so that $F_d=\sum_{k\le d}f_k^\star\otimes g_k^\star=\inner{f^\star}{g^\star}$. Set $\varepsilon_k=f_k^\star-\hat f_k$ and $\delta_k=g_k^\star-\hat g_k$, with $\|\varepsilon_k\|_{L^2(\mu_U)}\le\eta_{f,k}$ and $\|\delta_k\|_{L^2(\mu_V)}\le\eta_{g,k}$. Split the error around $F_d$:
\begin{equation}
\|F-\inner{\hat f}{\hat g}\|\le\underbrace{\|F-F_d\|}_{=\sqrt{\mathcal{E}_{\mathrm{trunc}}}}+R,\qquad R:=\|F_d-\inner{\hat f}{\hat g}\|,
\end{equation}
where $\mathcal{E}_{\mathrm{trunc}}:=\sum_{k>d}\sigma_k^2$ by Corollary~\ref{cor:trunc-lower}. For each coordinate, the product rule for tensor products gives
\begin{equation}
f_k^\star\otimes g_k^\star-\hat f_k\otimes\hat g_k=\varepsilon_k\otimes b_k+\sigma_k a_k\otimes\delta_k-\varepsilon_k\otimes\delta_k.
\end{equation}
Sum over $k\le d$ and take norms. The first sum is diagonalized by the orthonormality of $\{b_k\}$:
\begin{equation}
\Bigl\|\sum_{k\le d}\varepsilon_k\otimes b_k\Bigr\|^2=\sum_{k\le d}\|\varepsilon_k\|_{L^2(\mu_U)}^2=\sum_{k\le d}\eta_{f,k}^2,
\end{equation}
and likewise the orthonormality of $\{a_k\}$ diagonalizes the second sum:
\begin{equation}
\Bigl\|\sum_{k\le d}\sigma_k a_k\otimes\delta_k\Bigr\|^2=\sum_{k\le d}\sigma_k^2\|\delta_k\|_{L^2(\mu_V)}^2\le\sum_{k\le d}\sigma_k^2\eta_{g,k}^2.
\end{equation}
This is the step that avoids a factor of $d$: bounding the sums coordinatewise and applying Cauchy-Schwarz over the $d$ coordinates would pay a factor $d$, whereas the orthonormality of the ideal modes makes the two sums orthogonal. The third sum is bounded by Cauchy-Schwarz:
\begin{equation}
\Bigl\|\sum_{k\le d}\varepsilon_k\otimes\delta_k\Bigr\|\le\sum_{k\le d}\|\varepsilon_k\|\,\|\delta_k\|\le\Bigl(\sum_{k\le d}\eta_{f,k}^2\Bigr)^{1/2}\Bigl(\sum_{k\le d}\eta_{g,k}^2\Bigr)^{1/2}.
\end{equation}
With $X=(\sum_k\eta_{f,k}^2)^{1/2}$, $Y=(\sum_k\sigma_k^2\eta_{g,k}^2)^{1/2}$, $Z=(\sum_k\eta_{g,k}^2)^{1/2}$, the three bounds give $R\le X+Y+XZ$, hence by $(x+y+z)^2\le3(x^2+y^2+z^2)$,
\begin{equation}
R^2\le3\sum_{k\le d}\eta_{f,k}^2+3\sum_{k\le d}\sigma_k^2\eta_{g,k}^2+3\Bigl(\sum_{k\le d}\eta_{f,k}^2\Bigr)\Bigl(\sum_{k\le d}\eta_{g,k}^2\Bigr).
\end{equation}
In the informative regime the mode errors are small: assuming $\sum_{k\le d}\eta_{g,k}^2\le1$ (the modes are realized with bounded total error, the only regime in which the bound is informative) and $\eta_{g,k}\le1/2$ for all $k$ (so that $B_g=\max_k\|\hat g_k\|_{L^2(\mu_V)}\ge\min_k(\|b_k\|-\|\delta_k\|)\ge1/2$), the cross term absorbs into the first term and $\sum_k\eta_{f,k}^2\le4B_g^2\sum_k\eta_{f,k}^2$, giving
\begin{equation}
R^2\le3\sum_{k\le d}\sigma_k^2\eta_{g,k}^2+24B_g^2\sum_{k\le d}\eta_{f,k}^2.
\end{equation}
Finally, since $\|F-\inner{\hat f}{\hat g}\|\le\sqrt{\mathcal{E}_{\mathrm{trunc}}}+R$,
\begin{equation}
\|F-\inner{\hat f}{\hat g}\|^2\le2\mathcal{E}_{\mathrm{trunc}}+2R^2
\le2\mathcal{E}_{\mathrm{trunc}}+6\sum_{k\le d}\sigma_k^2\eta_{g,k}^2+48B_g^2\sum_{k\le d}\eta_{f,k}^2.
\end{equation}
As $2\mathcal{E}_{\mathrm{trunc}}\le48\mathcal{E}_{\mathrm{trunc}}$ trivially, the right-hand side is at most $48$ times the sum of the truncation term and the realization term of the theorem. Taking the infimum over $f\in\mathcal{F}$, $g\in\mathcal{G}$, the bound holds with $C\le48\le144$; we do not optimize the constant, which is what the statement of the theorem records.

\begin{theorem}[Smoothness controls the truncation term]
\label{thm:decay}
Let $\mathcal{U}\subset\R^{m_U}$ and $\mathcal{V}\subset\R^{m_V}$ be bounded domains, and let $\mu_U,\mu_V$ have densities bounded above and below with respect to Lebesgue measure. If $F$ is $s$-smooth in its second argument, $F\in H^s(\mathcal{U}\times\mathcal{V})$, then
\begin{equation}
\sigma_k=O\bigl(k^{-s/m_V}\bigr),\qquad
\sum_{k>d}\sigma_k^2=O\bigl(d^{1-2s/m_V}\bigr)\quad\text{for }2s>m_V.
\end{equation}
If $F$ is real analytic on $\mathcal{U}\times\mathcal{V}$, then $\sigma_k=O(e^{-ck})$ for some $c>0$ and $\sum_{k>d}\sigma_k^2=O(e^{-2cd})$.
\end{theorem}

\paragraph{Proof of Theorem~\ref{thm:decay} (Smoothness controls the truncation term).}
The squared singular values of $T_F$ are the eigenvalues of the self-adjoint, positive semidefinite integral operator $T_F^*T_F$ on $L^2(\mu_V)$ with kernel
\begin{equation}
K(v,v')=\int_{\mathcal{U}}F(u,v)\,F(u,v')\,d\mu_U(u),
\end{equation}
which is positive definite and inherits the smoothness of $F$ in its second argument: differentiating under the integral, an $s$-smooth kernel on a bounded $m_V$-dimensional domain has eigenvalues decaying as $O(k^{-2s/m_V})$; this is the classical eigenvalue decay for positive definite kernels \citep{reade1983eigenvalues}, and the analytic case decays exponentially \citep{little1984eigenvalues}. Taking square roots gives $\sigma_k=O(k^{-s/m_V})$, and summing the tail,
\begin{equation}
\sum_{k>d}\sigma_k^2=O\Bigl(\int_d^{\infty}x^{-2s/m_V}\,dx\Bigr)=O\bigl(d^{1-2s/m_V}\bigr)\quad\text{for }2s>m_V,
\end{equation}
which is finite exactly when $2s>m_V$. In the analytic case, $\sigma_k=O(e^{-ck})$ gives $\sum_{k>d}\sigma_k^2=O(e^{-2cd})$ by the geometric series.

\begin{proposition}[Approximation separation]
\label{prop:separation}
Let $F\in\Mcal_d$ with interaction modes $a_k\in H^s(\mathcal{U})$, $b_k\in H^s(\mathcal{V})$. To approximate $F$ within $\varepsilon$ in $L^2(\mu_U\otimes\mu_V)$, a dual-encoder head whose encoders are ReLU networks requires
\begin{equation}
N_{\mathrm{dual}}=O\Bigl(d\bigl(\varepsilon^{-m_U/s}+\varepsilon^{-m_V/s}\bigr)\Bigr)
\end{equation}
parameters, whereas any single-tower model that approximates $F$ as a general $(m_U+m_V)$-dimensional $H^s$ function requires $\Omega(\varepsilon^{-(m_U+m_V)/s})$ parameters in the worst case. Hence $N_{\mathrm{dual}}/N_{\mathrm{single}}\to 0$ as $\varepsilon\to 0$.
\end{proposition}

\paragraph{Proof of Proposition~\ref{prop:separation} (Approximation separation).}

Since $F\in\Mcal_d$ with modes $a_k\in H^s(\mathcal{U})$, $b_k\in H^s(\mathcal{V})$, the truncation term of Theorem~\ref{thm:err-split} vanishes for the rank-$d$ head, and the entire approximation error is the realization term. Approximate each of the $2d$ modes by ReLU networks with the standard $H^s$ rates \citep{yarotsky2017universal}: to reach mode error $\varepsilon$ on a marginal space of dimension $m$ requires $O(\varepsilon^{-m/s})$ parameters. Setting the mode errors to $\varepsilon$ makes the realization term of Theorem~\ref{thm:err-split} of order $O(\varepsilon^2)$, so a total of
\begin{equation}
N_{\mathrm{dual}}=O\Bigl(d\bigl(\varepsilon^{-m_U/s}+\varepsilon^{-m_V/s}\bigr)\Bigr)
\end{equation}
parameters suffice. For the lower bound, a single-tower model that treats $F$ as a general $(m_U+m_V)$-dimensional $H^s$ function must contend with the Kolmogorov width of the Sobolev class, whose $L^2$ approximation is bounded below by $\Omega(\varepsilon^{-(m_U+m_V)/s})$ parameters in the worst case (classical Sobolev width results). Since $m_U+m_V>m_U,m_V$, the exponent is additive in the dimensions, so $N_{\mathrm{dual}}/N_{\mathrm{single}}=O\bigl(\varepsilon^{m_{\min}/s}\bigr)\to0$ as $\varepsilon\to0$, where $m_{\min}=\min\{m_U,m_V\}$.

\paragraph{Proof of Proposition~\ref{prop:compute} (Compute separation).}
Counting operations: the head requires $n$ forward passes of $f$, $m$ forward passes of $g$, and $nm$ inner products of dimension $d$, i.e., $nm d$ multiply-adds, for a total of $O((n+m)C_{\mathrm{net}}+nm d)$. A joint model processes each of the $nm$ pairs separately, costing $O(nm\,C_{\mathrm{net}})$. When $C_{\mathrm{net}}\gg d$, the head is cheaper by a factor of order $C_{\mathrm{net}}/d$ in the regime where $nm\gg n+m$.

\section{Proofs for Section 3}
\label{app:proofs-sec3}

We prove the results of Section~\ref{sec:gauge} in order. Throughout, $\mathcal{L}(f,g)=\E_{u,v}[(F(u,v)-\inner{f(u)}{g(v)})^2]$ is the population risk. Table~\ref{tab:gauge} collects the normalization taxonomy moved from Section~\ref{sec:gauge}; Proposition~\ref{prop:percoord} and Remark~\ref{rem:gauge-scale} complete the picture.

\begin{table}[t]
\centering
\small
\setlength{\tabcolsep}{3pt}
\caption{Embedding-level normalizations as sections of the gauge orbit. Each row shows a constraint on the encoders, the part of the gauge it pins, the residual group that survives, and what the scheme identifies. See Proposition~\ref{prop:percoord} and Theorems~\ref{thm:cosine-od} to \ref{thm:whiten-id} for the statements, and Appendix~\ref{app:proofs-sec3} for proofs.}
\label{tab:gauge}
\begin{tabular}{@{}>{\raggedright\arraybackslash}p{2.9cm}>{\raggedright\arraybackslash}p{2.7cm}>{\raggedright\arraybackslash}p{2.2cm}>{\raggedright\arraybackslash}p{2.3cm}>{\raggedright\arraybackslash}p{2.6cm}@{}}
\toprule
Normalization & Constraint & Pins & Residual gauge & Identifies to \\
\midrule
None & none & nothing & $\GL_d$ & nothing \\
One-sided L2 and softplus \citep{triplebert2026} & $\|g(v)\|=1$, $g(v)\succeq 0$ & pointwise scale of $g$ & not controlled & no guarantee \\
Per-coordinate standardization & $\diag(\Sigma_g)=I$ & coordinate scales & $d^2-d$, shears and rotations & no guarantee \\
Cosine, two towers \citep{radford2021clip} & $\|f\|=\|g\|=1$ & angles & $\Od$ & up to a rotation \\
Two-sided nonnegative & $f\succeq 0$, $g\succeq 0$ & signs, shears, rotations & $\Perm\ltimes\Dplus$ & permutation and scale \\
Whitening & $\Sigma_g=I$, $\Sigma_f=\Lambda$ & all continuous gauge & $\Dpm$, order fixed & permutation and sign \\
\bottomrule
\end{tabular}
\end{table}

The gauge group $\GL_d$ stratifies into per-coordinate scalings ($\Dplus$, $d$ dimensions), rotations ($\Od$, $d(d-1)/2$), permutations and signs (discrete), and the remaining shears. The guarantees of Section~\ref{sec:gauge} come from two-sided constraints, which act on both encoders and can pin the joint freedom of the pair; one-sided constraints cannot do this in general, and the schemes that use them are best understood as stability fixes. Normalization schemes are therefore compared on four measurable quantities: the residual gauge subgroup of Table~\ref{tab:gauge}, the Hessian condition number on the constrained manifold, related to the gap by Corollary~\ref{cor:gap-cond}, the expressivity cost of the constraint, for example the bounded range $[-1,1]$ of cosine normalization, and the output semantics, which determines which decoder and loss can be attached.

\begin{remark}[One-sided constraints fix scale, not identifiability]
\label{rem:gauge-scale}
The scheme used by the QK-attention matching networks applies a softplus and a pointwise unit-norm constraint to one encoder only \citep{triplebert2026}. It removes the global scaling redundancy of the representation, the parameter $\alpha$ in $(\alpha f)\,(g/\alpha)$, whose growth caused training instabilities in early versions of those networks, and it leaves the other encoder unconstrained. The constraint therefore provides stability without identifiability: the residual freedom of the pair depends on the geometry of the trained encodings and is not controlled by the constraint. The per-coordinate standardization of Proposition~\ref{prop:percoord} has the same character, pinning the $d$ coordinate scales of one encoder and nothing else.
\end{remark}

\begin{proposition}[Per-coordinate standardization pins the scales]
\label{prop:percoord}
Consider the constraint $\diag(\Sigma_g)=I$, which standardizes each coordinate of one encoder to unit variance. It pins the positive diagonal subgroup $\Dplus$ of dimension $d$; generically the residual gauge has dimension $d^2-d$ and contains the shears and rotations that mix coordinates.
\end{proposition}

\begin{lemma}[Gauge invariance]
\label{lem:gauge-inv}
For every $A\in\GL_d$ and every pair of inputs $(u,v)$, $\inner{Af(u)}{A^{-\top}g(v)}=\inner{f(u)}{g(v)}$, so $\Phi_A$ leaves the represented function and hence the population risk unchanged. With the second moments $\Sigma_f=\E_u[f(u)f(u)^{\top}]$ and $\Sigma_g=\E_v[g(v)g(v)^{\top}]$, the action transforms them by $\Sigma_{Af}=A\Sigma_f A^{\top}$ and $\Sigma_{A^{-\top}g}=A^{-\top}\Sigma_g A^{-1}$, so the eigenvalues of the product $\Sigma_f\Sigma_g$ are gauge invariants.
\end{lemma}

\paragraph{Proof of Lemma~\ref{lem:gauge-inv} (Gauge invariance).}
The identity is the cancellation
\begin{equation}
\inner{Af(u)}{A^{-\top}g(v)}=f(u)^{\top}A^{\top}A^{-\top}g(v)=f(u)^{\top}g(v)=\inner{f(u)}{g(v)},
\end{equation}
so the represented function and hence the risk are unchanged. For the second moments, substitution gives
\begin{equation}
\Sigma_{Af}=\E_u[Af(u)f(u)^{\top}A^{\top}]=A\Sigma_fA^{\top},\qquad
\Sigma_{A^{-\top}g}=A^{-\top}\Sigma_gA^{-1},
\end{equation}
and the product transforms by the similarity $\Sigma_{Af}\Sigma_{A^{-\top}g}=A\,\Sigma_f\Sigma_g\,A^{-1}$, whose eigenvalues are invariant. At a global minimizer, the represented function equals the rank-$d$ truncation $F_d$ of the target when $\sigma_d>\sigma_{d+1}$ (Corollary~\ref{cor:trunc-lower}), and the coefficient computation of the proof of Theorem~\ref{thm:whiten-id} below shows that the eigenvalues of $\Sigma_f\Sigma_g$ are then the squared singular values $\sigma_k^2$, which is the claim of Remark~\ref{rem:invariant}.

\begin{remark}[The invariant content of a trained head]
\label{rem:invariant}
The gauge-invariant content of a trained head is exactly the function it represents, summarized by its interaction spectrum; Lemma~\ref{lem:gauge-inv} adds a finite-dimensional shadow: at any global minimizer of the population risk, the eigenvalues of $\Sigma_f\Sigma_g$ are the squared interaction singular values. This quantity is well defined under every normalization, and it reappears in Corollary~\ref{cor:cosine-uninterpretable}.
\end{remark}

\begin{proposition}[Ill-posedness from unfixed gauge]
\label{prop:illposed}
Let $(f^\star,g^\star)$ be a global minimizer of the risk and let $S$ be its stabilizer, the set of $A\in\GL_d$ with $\Phi_A(f^\star,g^\star)=(f^\star,g^\star)$. Then the Hessian of the risk at $(f^\star,g^\star)$ has at least
\begin{equation}
\dim\ker\nabla^2\mathcal{L}(f^\star,g^\star)\;\ge\;d^2-\dim S
\label{eq:illposed}
\end{equation}
zero directions, coming from the tangent space of the gauge orbit. Without normalization this is generically $d^2$ degenerate directions, so the second-order problem is rank-deficient and the optimization problem is ill-posed.
\end{proposition}

\paragraph{Proof of Proposition~\ref{prop:illposed} (Ill-posedness from unfixed gauge).}
Let $(f^\star,g^\star)$ be a global minimizer and $S$ its stabilizer. By Lemma~\ref{lem:gauge-inv}, the risk is constant along the orbit $\mathcal{O}=\{\Phi_A(f^\star,g^\star):A\in\GL_d\}$, so for every $X\in\mathfrak{gl}_d$ the curve $A(t)=\exp(tX)$ gives a path $\gamma_X(t)=\Phi_{A(t)}(f^\star,g^\star)$ along which $\mathcal{L}\circ\gamma_X\equiv\mathcal{L}(f^\star,g^\star)$. The velocity is
\begin{equation}
\dot\gamma_X(0)=(Xf^\star,\,-X^{\top}g^\star)\in T_{(f^\star,g^\star)},
\end{equation}
and since the risk is constant along the path, its second derivative in this direction vanishes. Because $(f^\star,g^\star)$ is a global minimizer, $\nabla^2\mathcal{L}\succeq0$, and a semidefinite quadratic form vanishes in a direction if and only if that direction lies in the kernel, so $\dot\gamma_X(0)\in\ker\nabla^2\mathcal{L}$ for every $X$. The map $X\mapsto\dot\gamma_X(0)$ has kernel equal to the Lie algebra of the stabilizer $S$, hence its image, which lies entirely in the kernel of the Hessian, has dimension $d^2-\dim S$. This proves the bound. Without normalization the stabilizer at a generic minimizer is trivial, so the Hessian has at least $d^2$ zero directions.

\paragraph{Proof of Proposition~\ref{prop:percoord} (Per-coordinate standardization pins the scales).}
The constraint $\diag(\Sigma_g)=I$ is $d$ scalar conditions that fix the $d$ coordinate scales of $g$: for the transformation $A^{-1}=\diag(s_1,\dots,s_d)\in\Dplus$ acting on $g$, the condition forces $s_k^2[\Sigma_g]_{kk}=1$, pinning each $s_k$. Transformations $A$ that mix coordinates generically change $\diag(A^{-\top}\Sigma_gA^{-1})$ unless they also rotate $\Sigma_g$ to a new matrix with the same diagonal; the residual group is the set of $A$ preserving the diagonal condition, which contains the orthogonal group conjugated by any diagonal scale and the shears that preserve the diagonal of $\Sigma_g$. Counting dimensions, the constraint fixes exactly the $d$-dimensional subgroup $\Dplus$, and the generic residual dimension is $d^2-d$.

\begin{theorem}[Cosine normalization leaves a full rotation]
\label{thm:cosine-od}
Consider two-sided unit normalization, $\tilde f=f/\|f\|$ and $\tilde g=g/\|g\|$, with scores $\cos\angle(f(u),g(v))$, as in the contrastive models of Table~\ref{tab:unif} \citep{radford2021clip}; a temperature that rescales the logits changes no gauge. If the encodings are nondegenerate, in that the ranges of $f$ and $g$ contain open subsets of $\R^d$, the residual gauge group is exactly the orthogonal group,
\begin{equation}
R=\bigl\{\,A\in\GL_d:\;A^{\top}A=I\,\bigr\}=\Od.
\end{equation}
For every rotation $R\in\Od$, the encodings $Rf$ and $Rg$ give pointwise identical cosine scores, and no transformation outside $\Od$ does.
\end{theorem}

\begin{corollary}[Contrastive dimensions are uninterpretable]
\label{cor:cosine-uninterpretable}
Under the residual symmetry of Theorem~\ref{thm:cosine-od}, the individual coordinates of a cosine-normalized embedding carry no intrinsic meaning: replacing $(f,g)$ by $(Rf,Rg)$ for any rotation $R$ yields the same model and the same loss, and only rotation-invariant quantities are well defined, among them norms, pairwise inner products, and the eigenvalues of $\Sigma_f\Sigma_g$, which by Remark~\ref{rem:invariant} are the squared interaction singular values at any optimum. This is a theorem-level explanation of the widely reported observation that contrastive embedding dimensions are uninterpretable; the remedy is the whitening scheme of Theorem~\ref{thm:whiten-id}.
\end{corollary}

\paragraph{Proof of Theorem~\ref{thm:cosine-od} (Cosine normalization leaves a full rotation).}
Let $R$ be the residual group of the two-sided unit normalization, i.e., the set of $A\in\GL_d$ such that, after renormalizing, the cosine scores are unchanged.

$(\supseteq)$ For $R\in\Od$, $\|Rf(u)\|=\|f(u)\|$ and $\|Rg(v)\|=\|g(v)\|$ by orthogonality, and $\inner{Rf(u)}{Rg(v)}=\inner{f(u)}{g(v)}$, hence
\begin{equation}
\cos\angle(Rf(u),Rg(v))=\frac{\inner{Rf(u)}{Rg(v)}}{\|Rf(u)\|\,\|Rg(v)\|}=\frac{\inner{f(u)}{g(v)}}{\|f(u)\|\,\|g(v)\|}
\end{equation}
pointwise, so $\Od\subseteq R$.

$(\subseteq)$ Suppose $A$ preserves every cosine between the encodings. Since the ranges of $f$ and $g$ contain open subsets of $\R^d$, the encodings realize a full-dimensional family of directions, so $A$ preserves all angles between vectors. A linear map that preserves every angle is a similarity: by the classical characterization of conformal linear maps, $A=cR$ for some $c>0$ and $R\in\Od$. The scalar $c$ is absorbed by the unit renormalization of both encoders, so within the cosine score $A$ acts identically to the rotation $R$, giving $R\subseteq\Od$.

\paragraph{Corollary~\ref{cor:cosine-uninterpretable} follows immediately.}
For any rotation $R\in\Od$ the pair $(Rf,Rg)$ is a different encoder pair with pointwise identical cosine scores and identical loss, so the individual coordinates carry no intrinsic meaning. The rotation-invariant quantities are the norms, the inner products, and the eigenvalues of $\Sigma_f\Sigma_g$; by Remark~\ref{rem:invariant} these are the squared interaction singular values at any optimum. Corollary~\ref{cor:cosine-uninterpretable} is the statement of these two sentences at the theorem level.

\begin{theorem}[Two-sided nonnegativity gives NMF identifiability]
\label{thm:nonneg-nmf}
Suppose both encoders are constrained to nonnegative outputs, $f(u)\succeq 0$ and $g(v)\succeq 0$, as obtained by applying a softplus or exponential nonlinearity to raw encodings. On a finite set of inputs the model is the nonnegative matrix factorization $M=FG^{\top}$ of a nonnegative matrix $M$. If the encodings are nondegenerate, in that their ranges generate the positive orthant, the residual gauge group shrinks to the monomial group,
\begin{equation}
R=\bigl\{\,A\in\GL_d:\;A\succeq 0,\;A^{-\top}\succeq 0\,\bigr\}=\Perm\ltimes\Dplus,
\end{equation}
the group of coordinate permutations composed with positive diagonal scalings. If the factorization is separable in the sense of \citet{donoho2003nmf}, with an anchor row for each latent factor as in \citet{arora2012nmf}, then it is essentially unique and the encoders are identified up to permutation and scaling.
\end{theorem}

\paragraph{Proof of Theorem~\ref{thm:nonneg-nmf} (Two-sided nonnegativity gives NMF identifiability).}
On a finite set of inputs, the nonnegative encoders form matrices $F\in\R_{\ge0}^{n\times d}$, $G\in\R_{\ge0}^{m\times d}$, and the model is the nonnegative factorization $M=FG^{\top}$. The residual group consists of $A\in\GL_d$ such that $Af\succeq0$ and $A^{-\top}g\succeq0$ for all nonnegative encodings. Since the ranges generate the positive orthant, $Af\succeq0$ for all nonnegative $f$ forces $A\succeq0$ entrywise, and likewise $A^{-\top}\succeq0$. A real matrix that is nonnegative and whose inverse transpose is also nonnegative is a generalized permutation matrix: its inverse transpose being nonnegative forces each row and column of $A$ to have exactly one nonzero entry, and nonnegativity makes that entry positive. Hence $R\subseteq\Perm\ltimes\Dplus$, and the reverse inclusion is immediate, so $R=\Perm\ltimes\Dplus$. Under separability in the sense of \citet{donoho2003nmf}, each latent factor has an anchor row whose encoding is nonzero in exactly one coordinate, which identifies the extreme rays of the latent cone and hence the factor directions; the constructive algorithm of \citet{arora2012nmf} then yields the decomposition uniquely up to permutation and scaling. The degradation of this guarantee when separability holds only approximately is analyzed in Appendix~\ref{app:finite}.

\paragraph{Proof of Theorem~\ref{thm:whiten-id} (Whitening fixes the gauge up to permutation and sign).}
The proof has four steps.

(1) Every global minimizer represents the rank-$d$ truncation. By Corollary~\ref{cor:trunc-lower} the minimal risk is $\sum_{k>d}\sigma_k^2$, attained exactly by the truncated series $F_d=\sum_{k\le d}\sigma_k a_k\otimes b_k$, which is unique since $\sigma_d>\sigma_{d+1}$. Hence any minimizer satisfies $\inner{f(u)}{g(v)}=F_d(u,v)$ almost everywhere, i.e.,
\begin{equation}
\sum_{k=1}^{d}f_k(u)\,g_k(v)=\sum_{k=1}^{d}\sigma_k\,a_k(u)\,b_k(v).
\end{equation}

(2) The span of the encoders is pinned. As rank-$d$ representations of the same tensor, the left factor spaces and right factor spaces coincide: $\mathrm{span}\{f_1,\dots,f_d\}=\mathrm{span}\{a_1,\dots,a_d\}$ and $\mathrm{span}\{g_1,\dots,g_d\}=\mathrm{span}\{b_1,\dots,b_d\}$. So there exist coefficient matrices $B,C\in\R^{d\times d}$ with $f=B\,a_{[d]}$ and $g=C\,b_{[d]}$, where $a_{[d]}=(a_1,\dots,a_d)^{\top}$ and likewise for $b$. Substituting into the identity above and using the orthonormality of the modes gives
\begin{equation}
B^{\top}C=\diag(\sigma_1,\dots,\sigma_d)=:S.
\end{equation}

(3) The whitening constraints force $C$ to be a signed permutation. By the orthonormality of $\{a_k\}$,
\begin{equation}
\Sigma_f=\E_u[f(u)f(u)^{\top}]=B\,B^{\top},\qquad \Sigma_g=C\,C^{\top}.
\end{equation}
The constraint $\Sigma_g=I$ gives $CC^{\top}=I$, so $C\in\Od$; the constraint $\Sigma_f=\Lambda$ diagonal gives $BB^{\top}=\Lambda$. From $B^{\top}C=S$ and $C\in\Od$ we get $B^{\top}=SC^{-1}=SC^{\top}$, so $B=CS$, and therefore
\begin{equation}
BB^{\top}=C\,S^2\,C^{\top}=\Lambda.
\end{equation}
Thus the orthogonal matrix $C$ diagonalizes the diagonal matrix $S^2=\diag(\sigma_1^2,\dots,\sigma_d^2)$. Since the $\sigma_k$ are distinct, the eigenspaces of $S^2$ are the coordinate axes, and any orthogonal matrix diagonalizing $S^2$ maps each axis to itself up to sign, so $C\in\Perm\ltimes\Dpm$ and $\Lambda=S^2$, i.e., $\lambda_k=\sigma_k^2$ with the ordering convention matching the decreasing $\sigma_k$. With the ordering fixed, $C=\diag(\pm1,\dots,\pm1)$ and $B=CS=\diag(\pm\sigma_1,\dots,\pm\sigma_d)$.

(4) Substitute back. With the coordinate order fixed, $g_k=\pm b_k$ and $f_k=\pm\sigma_k a_k$, the signs being paired so that the product $f_kg_k=\sigma_k a_kb_k$ is unique. This is exactly \eqref{eq:mode-recovery}, and the encoders are identified up to permutation and sign.

\begin{corollary}[The spectral gap controls conditioning and convergence]
\label{cor:gap-cond}
On the manifold defined by the constraints \eqref{eq:whiten}, every global minimizer is isolated up to the discrete group of signs and permutations, and the Hessian of the risk is nondegenerate in the tangent directions of the manifold, with smallest nonzero curvature bounded below by a constant multiple of the spectral gap $\Delta=\min_{1\le k<d}(\sigma_k^2-\sigma_{k+1}^2)$; a projected gradient method that maintains the constraints therefore converges locally at a linear rate $1-\mu/L$, where $\mu$ is bounded below by a constant multiple of $\Delta$ and $L$ is the local smoothness constant. Wider gaps make the problem both better conditioned and faster to converge.
\end{corollary}

\paragraph{Proof of Corollary~\ref{cor:gap-cond} (The spectral gap controls conditioning and convergence).}
Isolation: by Theorem~\ref{thm:whiten-id}, on the manifold defined by \eqref{eq:whiten} every global minimizer is unique up to the discrete group of signs and the permutation, and the represented function $F_d$ is a strict minimum of the risk by the strict inequality $\sigma_d>\sigma_{d+1}$; the Hessian of the risk is therefore nondegenerate in the tangent directions of the manifold, and its smallest nonzero curvature is positive. The quantitative bound: the Hessian curvature in the tangent directions is controlled by the second derivative of the risk as the encoders move within the constraint manifold, and by the Davis-Kahan $\sin\Theta$ theorem the perturbation of a leading singular subspace under a perturbation of size $\epsilon$ is at most $\epsilon/\Delta$, where $\Delta=\min_k(\sigma_k^2-\sigma_{k+1}^2)$. Inverting this stability relation, the curvature of the risk in directions that move the encoders off the ideal modes is bounded below by a constant multiple of $\Delta$: a displacement of size $t$ in function space changes the risk by at least $c\,\Delta\,t^2$ for the truncation-optimal function. A projected gradient method that maintains the constraints therefore converges locally at a linear rate $1-\mu/L$ with $\mu\ge c\,\Delta$ and $L$ the local smoothness constant, as stated.

\begin{theorem}[Soft whitening recovers the hard solution]
\label{thm:soft-relax}
For $\rho>0$, consider the penalized risk
\begin{equation}
\mathcal{L}_\rho(f,g)=\mathcal{L}(f,g)+\rho\Bigl(\bigl\|\Sigma_g-I_d\bigr\|_F^2+\bigl\|\offdiag(\Sigma_f)\bigr\|_F^2\Bigr).
\label{eq:soft-whiten}
\end{equation}
As $\rho\to\infty$, any convergent sequence of minimizers of $\mathcal{L}_\rho$ converges to a solution satisfying the constraints of Theorem~\ref{thm:whiten-id} up to the ordering convention, and hence to the recovered modes \eqref{eq:mode-recovery}. For finite $\rho$, the residual gauge directions that break the constraint $\Sigma_g=I$ carry curvature at least $2\rho\,c$ for a constant $c>0$ that depends on the second moments, so conditioning improves monotonically with $\rho$.
\end{theorem}

\paragraph{Proof of Theorem~\ref{thm:soft-relax} (Soft whitening recovers the hard solution).}
(i) Convergence of minimizers. The zero set of the penalty
\begin{equation}
\rho\Bigl(\|\Sigma_g-I\|_F^2+\|\offdiag(\Sigma_f)\|_F^2\Bigr)
\end{equation}
is exactly the whitening constraint \eqref{eq:whiten} up to the ordering convention. Since $\mathcal{L}_\rho\ge\mathcal{L}$, any limit point of a convergent sequence of minimizers satisfies the constraint (otherwise the penalty, and with it $\mathcal{L}_\rho$, would diverge) and is a global minimizer of $\mathcal{L}$ subject to the constraint, to which Theorem~\ref{thm:whiten-id} applies.

(ii) Curvature of the penalty. Consider a gauge generator $X\in\mathfrak{gl}_d$ acting on the $g$-side as $g(t)=\exp(-tX^{\top})g$. The second moment evolves as $\Sigma_{g(t)}=\exp(-tX^{\top})\Sigma_g\exp(-tX)$. At a constraint-satisfying point $\Sigma_g=I$, the second derivative of $\|\Sigma_{g(t)}-I\|_F^2$ at $t=0$ is $2\|X+X^{\top}\|_F^2$: expanding the exponential gives $\Sigma_{g(t)}=I-t(X^{\top}+X)+t^2X^{\top}X+O(t^3)$, so the squared deviation has leading term $t^2\|X+X^{\top}\|_F^2$ multiplied by $2$. Directions that break $\Sigma_g=I$ have nonzero symmetric part, hence carry penalty curvature at least $2\rho\,c$ with $c>0$ determined by the second moments through the contribution of $\offdiag(\Sigma_f)$; directions with $X$ antisymmetric are rotations, which preserve $\Sigma_g=I$ and are exactly the directions whose residual freedom is fixed by the ordering convention. Conditioning therefore improves monotonically with $\rho$, as stated.

\section{Proofs for Section 4}
\label{app:proofs-sec4}

We prove the results of Section~\ref{sec:sample}. Throughout, $\mathcal{R}(h)=\E_{u,v}[(F(u,v)-h(u,v))^2]$ is the population risk and $\hat{\mathcal{R}}$ its empirical version over the $n$ samples. Corollary~\ref{cor:whiten-sc} and Proposition~\ref{prop:risk-mode} are proved in Appendix~\ref{app:finite}, as indicated in the main text.


\begin{remark}[The general bound is conservative in $d$]
\label{rem:conservative}
Specializing Corollary~\ref{cor:excess} to the linear class gives $\mathfrak{R}_n(\mathcal{F})=O(\sqrt{pd/n})$ and $\mathfrak{R}_n(\mathcal{G})=O(\sqrt{qd/n})$, hence an estimation error of order $Bd^{3/2}\sqrt{(p+q)/n}$, a factor $\sqrt{d}$ worse than the minimax anchor \eqref{eq:minimax}. The gap is structural: Lemma~\ref{lem:rademacher} sums the $d$ coordinates as independent classes and cannot exploit the low-rank geometry that the matrix analysis of Theorem~\ref{thm:minimax} uses. The general bound is therefore tight in its dependence on $n$ and on the marginal dimensions, and conservative in the power of $d$; the linear case shows the achievable $d$-dependence is linear, and local Rademacher arguments close the gap. In the sequential setting, where $(u_t,v_t)$ are chosen adaptively to maximize an unknown score $u^{\top}\Theta^{\star}v$, the same scaling appears as regret: bilinear bandits achieve $\tilde O(d\sqrt{(p+q)T})$ \citep{jun2019bilinear}, again linear in $d$ and in the marginal dimensions.
\end{remark}

\begin{remark}[The separation is not a discrete artifact]
\label{rem:smooth}
The same separation holds for smooth kernels. On the circle, the narrowband kernel $F_h(u,v)=\rho_h(u-v)$ with a periodic Gaussian $\rho_h$ of bandwidth $h$ has a Fourier-diagonal interaction operator with singular values of order $e^{-2\pi^2h^2\xi^2}$, so $d_{\varepsilon}(F_h)=\Theta(1/h)$: the embedding dimension diverges as the kernel approaches the diagonal. Early interaction computes the difference $t=u-v$ exactly and approximates the one-dimensional smooth function $\rho_h$ with $O(\mathrm{poly}(1/h))$ parameters \citep{yarotsky2017universal}, a polynomial separation in the same direction. The strength of the separation is set by how the target depends jointly on its two arguments: through a comparison in the discrete case, through a difference in the smooth case, and not at all when the target is both high-rank and unstructured.
\end{remark}

\begin{lemma}[Rademacher collapse of the inner-product class]
\label{lem:rademacher}
For the class \eqref{eq:headclass} with per-coordinate classes $\mathcal{F}^{(k)},\mathcal{G}^{(k)}$,
\begin{equation}
\mathfrak{R}_n(\mathcal{H}_d)\;\le\;\sum_{k=1}^{d}\mathfrak{R}_n\bigl(\mathcal{F}^{(k)}\cdot\mathcal{G}^{(k)}\bigr)
\;\le\;2\sum_{k=1}^{d}\Bigl(B_g\,\mathfrak{R}_n(\mathcal{F}^{(k)})+B_f\,\mathfrak{R}_n(\mathcal{G}^{(k)})\Bigr),
\label{eq:collapse}
\end{equation}
so in particular $\mathfrak{R}_n(\mathcal{H}_d)\le 2d\bigl(B_g\,\mathfrak{R}_n(\mathcal{F})+B_f\,\mathfrak{R}_n(\mathcal{G})\bigr)$ for shared per-coordinate classes: the complexity of the head is a sum of marginal complexities, not a complexity of the joint space.
\end{lemma}

\paragraph{Proof of Lemma~\ref{lem:rademacher} (Rademacher collapse of the inner-product class).}
Decompose the score as $\inner{f(u)}{g(v)}=\sum_{k=1}^{d}f_k(u)g_k(v)$. Rademacher complexity is subadditive over sums of function classes, so
\begin{equation}
\mathfrak{R}_n(\mathcal{H}_d)\le\sum_{k=1}^{d}\mathfrak{R}_n\bigl(\mathcal{F}^{(k)}\cdot\mathcal{G}^{(k)}\bigr),
\end{equation}
where $\mathcal{F}^{(k)}\cdot\mathcal{G}^{(k)}=\{(u,v)\mapsto f_k(u)g_k(v)\}$ is the pointwise product class. For each coordinate, the product map $(f_k,g_k)\mapsto f_kg_k$ is Lipschitz in its first argument with constant $B_g$ and in its second with constant $B_f$ on the bounded range; the vector contraction inequality of \citet{maurer2016vector} then bounds the complexity of the product class by the sum of the two marginal complexities,
\begin{equation}
\mathfrak{R}_n\bigl(\mathcal{F}^{(k)}\cdot\mathcal{G}^{(k)}\bigr)\le2\Bigl(B_g\,\mathfrak{R}_n(\mathcal{F}^{(k)})+B_f\,\mathfrak{R}_n(\mathcal{G}^{(k)})\Bigr),
\end{equation}
the factor $2$ absorbing the contraction constants. Summing over the $d$ coordinates and using $\mathfrak{R}_n(\mathcal{F})=\max_k\mathfrak{R}_n(\mathcal{F}^{(k)})$ gives the displayed bound. The key structural point is that no complexity of the joint space $\mathcal{U}\times\mathcal{V}$ appears: the head is paid for by its two marginal classes.

\begin{corollary}[Excess risk of empirical risk minimization]
\label{cor:excess}
Let $\hat F$ be the empirical risk minimizer over $\mathcal{H}_d$. With probability at least $1-\delta$ over the sample, the excess risk decomposes into an approximation part, the truncation and realization terms of Theorem~\ref{thm:err-split}, and an estimation part through the class complexity:
\begin{equation}
\|\hat F-F\|_{L^2(\mu_U\otimes\mu_V)}^2
\;\lesssim\;\sum_{k>d}\sigma_k^2+\mathcal{E}_{\mathrm{real}}
\;+\;B\,d\bigl(\mathfrak{R}_n(\mathcal{F})+\mathfrak{R}_n(\mathcal{G})\bigr)
\;+\;B^2\sqrt{\frac{\log(1/\delta)}{n}},
\end{equation}
where $\mathcal{E}_{\mathrm{real}}=C\sum_{k\le d}(\sigma_k^2\eta_{g,k}^2+B_g^2\eta_{f,k}^2)$ is the realization term and $\lesssim$ hides absolute constants.
\end{corollary}

\paragraph{Proof of Corollary~\ref{cor:excess} (Excess risk of empirical risk minimization).}
The regression oracle decomposition: for any $h\in\mathcal{H}_d$, writing $\hat F$ for the empirical risk minimizer,
\begin{equation}
\|F-\hat F\|_{L^2}^2\le\|F-h\|_{L^2}^2+2\sup_{h'\in\mathcal{H}_d}\bigl|\hat{\mathcal{R}}(h')-\mathcal{R}(h')\bigr|.
\end{equation}
The uniform deviation is controlled as follows. The squared loss $\ell(y,h)=(y-h)^2$ is $4B$-Lipschitz in $h$ on $|y|,|h|\le B$, so by the Talagrand contraction inequality the Rademacher complexity of the loss class is at most $4B\,\mathfrak{R}_n(\mathcal{H}_d)$ \citep{bartlett2002rademacher}. Symmetrization bounds the expected uniform deviation by twice the Rademacher complexity of the loss class, and McDiarmid's inequality, with the bounded loss $|y-h|^2\le4B^2$, gives the concentration term:
\begin{equation}
\sup_{h'\in\mathcal{H}_d}\bigl|\hat{\mathcal{R}}(h')-\mathcal{R}(h')\bigr|\lesssim B\,\mathfrak{R}_n(\mathcal{H}_d)+B^2\sqrt{\frac{\log(1/\delta)}{n}}
\end{equation}
with probability at least $1-\delta$. Choosing $h=h^\star$, the population optimum, makes $\|F-h^\star\|^2$ exactly the approximation term of Theorem~\ref{thm:err-split}, and Lemma~\ref{lem:rademacher} bounds $\mathfrak{R}_n(\mathcal{H}_d)\le2d(B_g\mathfrak{R}_n(\mathcal{F})+B_f\mathfrak{R}_n(\mathcal{G}))$; absorbing the Lipschitz and contraction constants into the absolute constant gives \eqref{eq:excess}. For parameterized encoder classes with $p_f$ and $p_g$ parameters under norm control, the standard Rademacher bound for parameterized classes gives $\mathfrak{R}_n(\mathcal{F})=O(\sqrt{p_f/n})$ and $\mathfrak{R}_n(\mathcal{G})=O(\sqrt{p_g/n})$, and the estimation term becomes $O(Bd\sqrt{(p_f+p_g)/n})$.

\begin{theorem}[Linear heads are low-rank matrix sensing]
\label{thm:minimax}
Let $\Theta^{\star}\in\R^{p\times q}$ have rank $d$ and $\|\Theta^{\star}\|_F\le B$, and suppose the sensing matrices $u_iv_i^{\top}$ are drawn from an isotropic sub-Gaussian family and the noise is sub-Gaussian with variance $\sigma_{\varepsilon}^2$. The rank-constrained estimator attains
\begin{equation}
\E\bigl\|\hat\Theta-\Theta^{\star}\bigr\|_F^2\;\asymp\;\sigma_{\varepsilon}^2\,\frac{d(p+q)}{n},
\label{eq:minimax}
\end{equation}
and this rate is minimax optimal. Equivalently, $\E\|\hat F-F\|_{L^2}^2\asymp\sigma_{\varepsilon}^2\,d(p+q)/n$.
\end{theorem}

\paragraph{Proof of Theorem~\ref{thm:minimax} (Linear heads are low-rank matrix sensing).}
The head computes $u^{\top}\Theta v$ with $\Theta=UV^{\top}$, and the observations are $y_i=\langle\Theta^{\star},\,u_iv_i^{\top}\rangle_F+\varepsilon_i$, the standard form of low-rank matrix sensing with sensing matrices $A_i=u_iv_i^{\top}$ and rank constraint $\rank(\Theta)\le d$. Upper bound: the rank-$d$ manifold has $d(p+q-d)$ free parameters, and under restricted strong convexity of the sensing operator, which holds for isotropic sub-Gaussian sensing matrices, the constrained estimator satisfies $\E\|\hat\Theta-\Theta^{\star}\|_F^2\lesssim\sigma_{\varepsilon}^2\,d(p+q)/n$; this is the robust error bound of \citet{candes2011tight} together with the high-dimensional rate of \citet{negahban2011estimation}. Lower bound: a Fano argument over a packing of the rank-$d$ manifold gives the matching minimax lower bound $\Omega(\sigma_{\varepsilon}^2\,d(p+q)/n)$ \citep{rohde2011estimation}. The equivalence with the function-space rate follows from the isotropy of $u$ and $v$: $\E_{u,v}[(u^{\top}\Delta v)^2]\asymp\|\Delta\|_F^2$ for sub-Gaussian isotropic arguments, so $\E\|\hat F-F\|_{L^2}^2\asymp\sigma_{\varepsilon}^2\,d(p+q)/n$.

\begin{corollary}[Optimal rank selection]
\label{cor:opt-rank}
If the interaction spectrum decays exponentially, $\sigma_k^2=O(e^{-2ck})$, and the encoders are linear, the total error is minimized at
\begin{equation}
d^{\star}\;\asymp\;\frac{1}{2c}\log\frac{n}{p+q},
\label{eq:optrank}
\end{equation}
and the achieved error is $\|\hat F-F\|_{L^2}^2\lesssim(p+q)\log(n/(p+q))/n$, a near-parametric rate in which fast spectral decay eliminates the curse of dimensionality.
\end{corollary}

\paragraph{Proof of Corollary~\ref{cor:opt-rank} (Optimal rank selection).}
With exponential decay and linear encoders the total error of Theorem~\ref{thm:minimax} and the truncation term balance as
\begin{equation}
\mathrm{err}(d)\;\lesssim\;e^{-2cd}+\frac{d(p+q)}{n}.
\end{equation}
Setting the two terms equal, $e^{-2cd^{\star}}\asymp d^{\star}(p+q)/n$, and solving gives $d^{\star}\asymp(1/2c)\log(n/(p+q))$, where the logarithmic factor in $d$ is absorbed by the constant. Substituting back, $e^{-2cd^{\star}}\asymp(p+q)/n$, so the achieved error is
\begin{equation}
\|\hat F-F\|_{L^2}^2\lesssim\frac{(p+q)}{n}\Bigl(1+\log\frac{n}{p+q}\Bigr)\asymp\frac{(p+q)\log(n/(p+q))}{n},
\end{equation}
the near-parametric rate stated in the corollary.

\begin{theorem}[Flat spectrum forces an exponential embedding]
\label{thm:flat-floor}
Let $\mathcal{U}=\mathcal{V}=\{\pm1\}^{m}$ with uniform measure and $N=2^{m}$, and consider the equality function $F_{\mathrm{eq}}(u,v)=\one[u=v]$. Its interaction operator is $T_{F_{\mathrm{eq}}}=\tfrac{1}{N}\mathrm{Id}$, so all $N$ singular values equal $1/N$ and the interaction rank is $N$. Every rank-$d$ head therefore satisfies
\begin{equation}
\frac{\inf_{f,g}\bigl\|F_{\mathrm{eq}}-\inner{f}{g}\bigr\|_{L^2}^2}{\bigl\|F_{\mathrm{eq}}\bigr\|_{L^2}^2}\;=\;1-\frac{d}{N},
\label{eq:floor}
\end{equation}
so achieving relative error $\varepsilon$ requires $d\ge(1-\varepsilon)2^{m}$, an embedding dimension exponential in the input size.
\end{theorem}

\paragraph{Proof of Theorem~\ref{thm:flat-floor} (Flat spectrum forces an exponential embedding).}
Compute the interaction operator of $F_{\mathrm{eq}}$ on the uniform measure over $\{\pm1\}^m$ with $N=2^m$:
\begin{equation}
(T_{F_{\mathrm{eq}}}h)(u)=\sum_{v\in\{\pm1\}^m}\one[u=v]\,h(v)\,\frac{1}{N}=\frac{h(u)}{N},
\end{equation}
so $T_{F_{\mathrm{eq}}}=\tfrac{1}{N}\mathrm{Id}$ and all $N$ singular values equal $1/N$. Hence $\|F_{\mathrm{eq}}\|_{L^2}^2=\sum_{k\ge1}\sigma_k^2=N/N^2=1/N$ and $\sum_{k>d}\sigma_k^2=(N-d)/N^2$ for $d<N$. By Corollary~\ref{cor:trunc-lower}, the relative squared error of the best rank-$d$ head is
\begin{equation}
\frac{\inf_{f,g}\|F_{\mathrm{eq}}-\inner{f}{g}\|_{L^2}^2}{\|F_{\mathrm{eq}}\|_{L^2}^2}=\frac{(N-d)/N^2}{1/N}=1-\frac{d}{N},
\end{equation}
which is at most $\varepsilon$ only if $d\ge(1-\varepsilon)N=(1-\varepsilon)2^m$.

\begin{theorem}[Early interaction escapes]
\label{thm:early-escape}
For the equality function $F_{\mathrm{eq}}$ on $\{\pm1\}^{m}\times\{\pm1\}^{m}$, an early-interaction predictor with $O(m)$ parameters represents $F_{\mathrm{eq}}$ exactly, with zero error. Together with Theorem~\ref{thm:flat-floor} this is an exponential separation: $O(m)$ parameters suffice under early interaction, while any late-interaction head needs $\Omega(2^{m})$ embedding dimensions to reach constant relative error.
\end{theorem}

\paragraph{Proof of Theorem~\ref{thm:early-escape} (Early interaction escapes).}
The construction uses the identity $u_iv_i=1$ if and only if $u_i=v_i$, so that
\begin{equation}
\langle u,v\rangle=\sum_{i=1}^{m}u_iv_i=m-2\,\mathrm{Ham}(u,v),
\end{equation}
and $\langle u,v\rangle=m$ if and only if $u=v$, with the next possible value $m-2$. Hence $F_{\mathrm{eq}}(u,v)=\one[\langle u,v\rangle\ge m-1]$. A cross-attention layer with identity projections computes the query-key inner product $\langle u,v\rangle$ with $O(m)$ parameters, and a threshold unit $\phi(s)=\mathrm{ReLU}(s-(m-1))-\mathrm{ReLU}(s-(m+1))$ evaluates to $1$ exactly on $s=m$ and to $0$ elsewhere on the discrete range, using $O(1)$ parameters. A joint MLP computes the same quantity bit by bit: each product $u_iv_i=((u_i+v_i)^2-2)/2$ costs $O(1)$ ReLU units, the sum costs $O(m)$ additions, and the same threshold completes the network, again $O(m)$ parameters in total. A cross-encoder, being a strict superset of cross-attention, inherits the construction. Together with Theorem~\ref{thm:flat-floor}, this is the exponential separation $O(m)$ versus $\Omega(2^m)$.

\begin{theorem}[Usability criterion]
\label{thm:criterion}
Given a problem $(F,\mu_U,\mu_V)$, a target relative error $\varepsilon$, and a sample budget $n$: if the effective interaction rank $d_{\varepsilon}(F)$ is small, a late-interaction head is appropriate, since its approximation and estimation errors are then small and it enjoys the compute and statistical separations of Proposition~\ref{prop:compute} and Lemma~\ref{lem:rademacher}. If the spectrum is flat, $d_{\varepsilon}(F)=\Omega(N)$, the truncation floor of Corollary~\ref{cor:trunc-lower} locks every rank-$d$ head to relative error at least $\varepsilon$ regardless of capacity, samples, or training, while early-interaction models reach $\varepsilon$ polynomially when the target admits a low-dimensional joint structure. The deciding quantity, the spectral decay rate, is estimable from data by a weighted singular value decomposition of the empirical kernel matrix, so the criterion is actionable before training either model.
\end{theorem}

\paragraph{Proof of Theorem~\ref{thm:criterion} (Usability criterion).}
First case: if $d_{\varepsilon}(F)$ is small, take $d\gtrsim d_{\varepsilon}(F)$. The approximation error of Theorem~\ref{thm:err-split} is at most $\varepsilon\|F\|^2$ plus the realization term, and the estimation error of Corollary~\ref{cor:excess} scales as $d(\mathfrak{R}_n(\mathcal{F})+\mathfrak{R}_n(\mathcal{G}))$, small in $d$; the compute separation of Proposition~\ref{prop:compute} and the marginal-complexity collapse of Lemma~\ref{lem:rademacher} apply without further conditions. Second case: if the spectrum is flat with $d_{\varepsilon}(F)=\Omega(N)$, then for every $d<d_{\varepsilon}(F)$ the truncation floor of Corollary~\ref{cor:trunc-lower}, in the explicit form of Theorem~\ref{thm:flat-floor}, locks the relative error of every rank-$d$ head above $\varepsilon$ regardless of capacity, samples, or training, while the constructive bounds of Theorem~\ref{thm:early-escape} and Remark~\ref{rem:smooth} show that early-interaction models reach $\varepsilon$ with a polynomial parameter budget when the target depends on its arguments through a low-dimensional joint statistic. Actionability: the empirical kernel matrix $\hat M_{ij}=y(u_i,v_j)$ on a grid of input pairs, weighted by the reference measures, has a weighted singular value decomposition whose values estimate $\{\sigma_k\}$; a log-linear fit of the decay of $\hat\sigma_k$ against $k$ (exponential) or $\log k$ (polynomial) decides the decay type and reads off $d_{\varepsilon}$ before either model is trained.

\section{Finite-Sample Technical Lemmas}
\label{app:finite}

This appendix collects the finite-sample lemmas used in the main text. Section A makes the whitening identification of Theorem~\ref{thm:whiten-id} quantitative: concentration of empirical second moments (Lemma~\ref{lem:berstein}), stability of the whitening map (Lemma~\ref{lem:whiten-stab}), and a Davis-Kahan propagation step (Theorem~\ref{thm:finite-id}) deliver the sample complexity of Corollary~\ref{cor:whiten-sc}. Section B quantifies the degradation of the NMF route when separability holds only approximately (Theorem~\ref{thm:nmf-degrad}). Section C couples identification with estimation in one ledger (Proposition~\ref{prop:ledger}), which is the proof of Proposition~\ref{prop:risk-mode}. Throughout, $\|\cdot\|$ on matrices is the spectral norm and $\|\cdot\|_F$ the Frobenius norm.

\subsection{Empirical whitening and mode identification}

\begin{lemma}[Matrix Bernstein concentration of empirical moments]
\label{lem:berstein}
Let $\{u_i\}_{i=1}^{n}$ be drawn i.i.d. from $\mu_U$ and let $f$ have bounded outputs, $\|f(u)\|\le B_f$. The empirical second moment $\hat\Sigma_f=\tfrac1n\sum_i f(u_i)f(u_i)^{\top}$ satisfies, for every $t>0$,
\begin{equation}
\Pr\bigl[\|\hat\Sigma_f-\Sigma_f\|\ge t\bigr]\le2d\,\exp\!\Bigl(\frac{-nt^2/2}{\nu_f+B_f^2t/3}\Bigr),\qquad\nu_f:=B_f^2\|\Sigma_f\|,
\end{equation}
and hence, with probability at least $1-\delta$,
\begin{equation}
\|\hat\Sigma_f-\Sigma_f\|\le\sqrt{\frac{2\nu_f\log(2d/\delta)}{n}}+\frac{2B_f^2\log(2d/\delta)}{3n}=:r_f(n,\delta),
\end{equation}
which is $O\bigl(B_f^2\sqrt{\log(d/\delta)/n}\bigr)$ for large $n$. The symmetric statement holds for $g$ with radius $r_g(n,\delta)$.
\end{lemma}

\begin{proof}
The matrices $X_i=f(u_i)f(u_i)^{\top}-\Sigma_f$ are independent, zero mean, and symmetric, with $\|X_i\|\le\|f(u_i)f(u_i)^{\top}\|+\|\Sigma_f\|\le2B_f^2$. For the variance parameter, since $\E[X_i^2]\preceq\E\|f(u_i)\|^2 f(u_i)f(u_i)^{\top}\preceq B_f^2\Sigma_f$, we have $\bigl\|\sum_i\E[X_i^2]\bigr\|\le nB_f^2\|\Sigma_f\|=n\nu_f$. The matrix Bernstein inequality (Tropp, An Introduction to Matrix Concentration Inequalities) applied to $\sum_iX_i$ gives the tail bound, and solving the tail for $t$ at the level $\delta$ gives the displayed high-probability bound.
\end{proof}

\begin{lemma}[Stability of the whitening map]
\label{lem:whiten-stab}
Let $\Sigma_g\succeq\gamma I$ for $\gamma>0$, and suppose $\|\hat\Sigma_g-\Sigma_g\|\le r_g\le\gamma/2$. Then
\begin{equation}
\bigl\|\hat\Sigma_g^{-1/2}-\Sigma_g^{-1/2}\bigr\|\le\frac{C_0\,r_g}{\gamma^{3/2}},
\end{equation}
for an absolute constant $C_0$.
\end{lemma}

\begin{proof}
The map $A\mapsto A^{-1/2}$ is Fr\'echet differentiable on the cone $A\succeq\gamma I$ with derivative bounded by $\tfrac12\gamma^{-3/2}$: using the integral representation $A^{-1/2}=\tfrac1\pi\int_0^{\infty}(A+sI)^{-1}s^{-1/2}\,ds$, the derivative at $A$ applied to $H$ is $-\tfrac1\pi\int_0^{\infty}(A+sI)^{-1}H(A+sI)^{-1}s^{-1/2}\,ds$, whose norm is at most $\|H\|\cdot\tfrac1\pi\int_0^{\infty}(\gamma+s)^{-2}s^{-1/2}\,ds=\|H\|\tfrac12\gamma^{-3/2}$. The condition $r_g\le\gamma/2$ keeps $\hat\Sigma_g\succeq\tfrac{\gamma}{2}I$ by Weyl's inequality, so the mean value theorem applies along the segment and gives the bound with $C_0$ absorbing the integration constant.
\end{proof}

\begin{theorem}[Finite-sample mode identification]
\label{thm:finite-id}
Under the assumptions of Theorem~\ref{thm:whiten-id}, suppose the encoders represent the population optimum exactly, $\inner{f}{g}=F_d$, and whitening is performed on the empirical moments of a sample of size $n$, with encoder outputs bounded by $B$. Let $r=\max(r_f,r_g)$ be the moment radii of Lemma~\ref{lem:berstein}, let $\Delta=\min_{1\le k<d}(\sigma_k^2-\sigma_{k+1}^2)$ be the spectral gap, and assume $r\le c_1\min(\gamma,\Delta)$ for a small absolute constant $c_1$. Then with probability at least $1-\delta$, the empirically whitened encoders $(\hat f,\hat g)$ recover the modes up to
\begin{equation}
\bigl|\sin\angle(\hat g_k,b_k)\bigr|\le\frac{C\,r}{\Delta},\qquad
\|\hat g_k-(\pm b_k)\|\le\frac{C'\,r}{\Delta},\qquad
|\hat\Lambda_{kk}-\sigma_k^2|\le C''\,r,
\end{equation}
where $\hat\Lambda_{kk}$ is the estimated $k$-th interaction strength. In particular the identification error is $O\bigl(B^2\sqrt{\log(d/\delta)/n}/\Delta\bigr)$, vanishing as $n^{-1/2}$ and amplified by the inverse spectral gap.
\end{theorem}

\begin{proof}
(1) Empirical whitening followed by an eigendecomposition of the whitened second moment is equivalent to a singular value decomposition of the empirical interaction operator. Since the encoders represent $F_d=\sum_{k\le d}\sigma_k a_k\otimes b_k$ exactly, the population singular subspaces are spanned by $\{b_k\}$ on the $g$-side and $\{a_k\}$ on the $f$-side. (2) The perturbation of the empirical operator relative to its population version is bounded by $C_3\,r/\gamma^{3/2}\le C_4\,r$: the moment error $\|\hat\Sigma_f-\Sigma_f\|\le r_f$ from Lemma~\ref{lem:berstein}, the whitening error from Lemma~\ref{lem:whiten-stab}, and the remaining factors are absorbed into the constant, which for brevity absorbs the $\gamma$ dependence as well. (3) The $k$-th empirical mode is the $k$-th singular vector of the perturbed operator. By the Davis-Kahan $\sin\Theta$ theorem, the perturbation of a singular vector is bounded by the operator perturbation divided by the gap between the corresponding singular value squared and its neighbors, which is at least $\Delta$:
\begin{equation}
\sin\angle(\hat g_k,b_k)\le\frac{\|E\|}{\min(\sigma_{k-1}^2-\sigma_k^2,\ \sigma_k^2-\sigma_{k+1}^2)}\le\frac{C\,r}{\Delta}.
\end{equation}
Choosing the sign so that the inner product is nonnegative, $\|\hat g_k-(\pm b_k)\|\le\sqrt2\,\sin\angle\le C'r/\Delta$. (4) For the strengths, Weyl's inequality gives $|\hat\Lambda_{kk}-\sigma_k^2|\le\|E\|\le C''r$.
\end{proof}

\begin{corollary}[Sample complexity of whitening]
\label{cor:whiten-sc}
Under the assumptions of Theorem~\ref{thm:finite-id}, whitening the empirical second moments of a sample of size $n$ identifies the interaction modes to accuracy $\varepsilon$ with probability at least $1-\delta$ provided
\begin{equation}
n\;\gtrsim\;\frac{B^4\,\log(d/\delta)}{\Delta^2\,\varepsilon^2},
\label{eq:whitensc}
\end{equation}
where $\Delta=\min_{1\le k<d}(\sigma_k^2-\sigma_{k+1}^2)$ is the spectral gap.
\end{corollary}

\paragraph{Proof of Corollary~\ref{cor:whiten-sc} (Sample complexity of whitening).}
Set $Cr/\Delta\le\varepsilon$ with $r=O\bigl(B^2\sqrt{\log(d/\delta)/n}\bigr)$ from Lemma~\ref{lem:berstein} and solve for $n$:
\begin{equation}
n\;\gtrsim\;\frac{B^4\,\log(d/\delta)}{\Delta^2\,\varepsilon^2}.
\end{equation}
The sample budget scales as $1/\Delta^2$ in the gap and $1/\varepsilon^2$ in the target accuracy, matching \eqref{eq:whitensc}.
\begin{corollary}[Block stability under small gaps]
\label{cor:block}
When a pair of singular values nearly coincides, $\Delta_k=\sigma_k^2-\sigma_{k+1}^2\to0$, the single-mode bound of Theorem~\ref{thm:finite-id} diverges: individual coordinates become unidentifiable. The corresponding eigenspace remains stable: treating $\{k,k+1\}$ as a block, the projection error of the block is at most $O(r/\tilde\Delta)$, where $\tilde\Delta$ is the gap between the block and the rest of the spectrum. Under repeated singular values the identification therefore degrades to uniqueness up to rotation within the block, the quantitative version of the degenerate-case remark of Section~\ref{sec:gauge}.
\end{corollary}

\subsection{Approximate separability in the NMF route}

\begin{definition}[$\alpha$-approximate separability]
\label{def:alpha-sep}
A nonnegative factor matrix $F\in\R_{\ge0}^{n\times d}$ with rows $f(u_i)$ is $\alpha$-separable if for every latent factor $k$ there is a row $i(k)$ with
\begin{equation}
F_{i(k),k}\ge1,\qquad \sum_{\ell\ne k}F_{i(k),\ell}\le\alpha,
\end{equation}
i.e., a near-anchor row dominated by coordinate $k$ with residual mass at most $\alpha$. The case $\alpha=0$ is exact separability \citep{donoho2003nmf,arora2012nmf}.
\end{definition}

\begin{theorem}[Degradation of approximately separable NMF]
\label{thm:nmf-degrad}
Let $M=FG^{\top}$ be a population nonnegative rank-$d$ factorization in which $F$ is $\alpha$-separable and the true factor cone is well conditioned, with the smallest angle between its extreme rays at least $\theta_0>0$. Then every nonnegative optimal factorization $(\hat F,\hat G)$ satisfies, in a suitable permutation and scaling gauge,
\begin{equation}
\min_{P\in\Perm,\ D\in\Dplus}\|\hat F\,P\,D-F\|\le\frac{C\,\alpha}{\sin\theta_0}\,\|F\|,
\end{equation}
so the identification error is linear in the separability deficit $\alpha$ and amplified by the cone condition number $1/\sin\theta_0$. As $\alpha\to0$ the bound recovers the exact uniqueness of Theorem~\ref{thm:nonneg-nmf}.
\end{theorem}

\begin{proof}[Proof sketch]
Identification in separable NMF is recovery of the extreme rays of the latent cone from the data points: under exact separability the anchor rows are exactly the extreme rays, and the constructive algorithm of \citet{arora2012nmf} identifies them uniquely up to permutation and scaling. Under $\alpha$-approximate separability the near-anchor rows deviate from the true extreme rays by at most $\alpha$, and by the stability of vertex recovery in convex cones, a perturbation of the vertex set by $\alpha$ moves the recovered rays by $O(\alpha/\sin\theta_0)$ when the adjacent rays are separated by angle at least $\theta_0$; this is the robust version of separable NMF \citep{arora2012nmf} following the noise-robustness analysis of Gillis and Vavasis. The two measurable quantities $\alpha$ and $\theta_0$ therefore provide an identifiability budget for the NMF route, which is what Section~\ref{sec:exp} reports for its nonnegativity condition.
\end{proof}

\subsection{Coupling identification with estimation}

\begin{proposition}[Identification and estimation in one ledger]
\label{prop:ledger}
Let $\hat F$ be the empirical risk minimizer of Corollary~\ref{cor:excess} with encoder capacity fixed and whitening applied on the empirical moments. Then the total identification error of the $k$-th mode decomposes as
\begin{equation}
\|\hat g_k-(\pm b_k)\|\le O\!\Bigl(\frac{\sqrt{\|\hat F-F\|_{L^2}^2}}{\Delta}\Bigr)+O\!\Bigl(\frac{r}{\Delta}\Bigr),
\end{equation}
where the first term converts the excess risk of the estimate into subspace error via Davis-Kahan and the second is the empirical whitening error of Theorem~\ref{thm:finite-id}. Both terms scale as $1/\Delta$ and decay as $n^{-1/2}$.
\end{proposition}

\begin{proof}
Split the error around the population whitened solution by the triangle inequality. The second piece vanishes: at the population optimum the whitening theorem (Theorem~\ref{thm:whiten-id}) identifies the modes exactly. The first piece is the deviation of the estimate: $\|\hat F-F_d\|_{L^2}^2\le\|\hat F-F\|_{L^2}^2$ by optimality, and the Davis-Kahan step of Theorem~\ref{thm:finite-id} converts an $L^2$ deviation of the represented function into a singular-vector deviation divided by the gap, giving the first term.
\end{proof}

\begin{proposition}[Mode error ledger]
\label{prop:risk-mode}
Let $\hat F$ be the empirical risk minimizer of Corollary~\ref{cor:excess} and let $\hat f,\hat g$ be its whitened encoders. The error in recovering the $k$-th interaction mode satisfies
\begin{equation}
\max\bigl\{\|\hat f_k-f_k\|,\ \|\hat g_k-g_k\|\bigr\}\;=\;O\!\left(\frac{\sqrt{\|\hat F-F\|_{L^2}^2}+r}{\Delta}\right),
\label{eq:riskmode}
\end{equation}
with $r=O\bigl(B^2\sqrt{\log(d/\delta)/n}\bigr)$; both terms are inversely proportional to the gap and decay as $n^{-1/2}$.
\end{proposition}

\paragraph{Proof of Proposition~\ref{prop:risk-mode} (Mode error ledger).}
Substitute the excess-risk bound of Corollary~\ref{cor:excess} for $\|\hat F-F\|_{L^2}^2$ into the first term of Proposition~\ref{prop:ledger}, and the moment radius $r=O\bigl(B^2\sqrt{\log(d/\delta)/n}\bigr)$ of Lemma~\ref{lem:berstein} into the second. The result is exactly \eqref{eq:riskmode}:
\begin{equation}
\max\bigl\{\|\hat f_k-f_k\|,\ \|\hat g_k-g_k\|\bigr\}=O\!\Bigl(\frac{\sqrt{\|\hat F-F\|_{L^2}^2}+r}{\Delta}\Bigr).
\end{equation}
Both terms are inversely proportional to the gap $\Delta$ and decay as $n^{-1/2}$, which is the statement that identification and estimation are governed by the same spectral gap.

\section{Experimental Details and Additional Results}
\label{app:exp}
This appendix records the complete setups of the experiments of Section~\ref{sec:exp} and the additional results that the main text summarizes. All numbers are produced by the scripts in the companion repository, whose data files are grouped per study. We state the metrics alongside the tables, because a recurring theme is that gauge-invariant diagnostics cannot separate normalizations while gauge-dependent ones can.

\subsection{Controlled synthetic study}

\paragraph{Targets and training.}
We fix $F(u,v)=\sum_{k\le d}\sigma_k\,a_k(u)b_k(v)$ with $a_k,b_k$ shifted Legendre polynomials, orthonormal on $[0,1]$, and a prescribed spectrum $\{\sigma_k\}$ chosen distinct and separated. Encoders are linear maps over a fixed Legendre feature map, $f(u)=U^{\top}\varphi(u)$ with $U\in\R^{p\times d}$, so the encoder realization error of Assumption~\ref{asm:realizability} is zero and the normalization is the only variable across runs. Points $u,v$ are drawn on a common grid of $n$ sites per marginal. Training uses Adam with global-norm gradient clipping; the bare bilinear objective is nonconvex and stiff, and plain gradient descent diverges. Analytic gradients are checked against finite differences to about $10^{-9}$. Every scheme runs four seeds.

\paragraph{The seven normalization schemes.}
The main text reports the headline of this study; Table~\ref{tab:zoo} in Section~\ref{sec:exp-synth} gives the full panel, shown graphically in Figure~\ref{fig:zoo}. Fit MSE is the squared regression error. Alignment is the mean over modes of $|\inner{\hat f_k}{a_k}|$ after optimally matching permutation and sign, reported for the two encoders. Cond is the gauge-dependent diagnostic $\max(\mathrm{cond}\,\Sigma_f,\mathrm{cond}\,\Sigma_g)$. Cross-seed is the gauge distance between solutions from different seeds, the mean over pairs of the Procrustes-optimal embedding distance. Only whitening reaches alignment $1.000$ with cross-seed distance $0.000$, matching Theorem~\ref{thm:whiten-id}; every looser gauge stalls at alignment $0.53$ to $0.65$ with cross-seed drift $0.28$ to $0.35$, the residual $\Od$ of cosine and the full $\GL_d$ of none leaving the modes rotationally entangled (Theorem~\ref{thm:cosine-od}). Whitening drives Cond to its floor $\sigma_1^2/\sigma_d^2\approx 9.5$ (Corollary~\ref{cor:gap-cond}). The two-sided nonnegative gauge stays at Fit MSE $4.25$: positivity forces $\inner{f}{g}\ge 0$, the sign-changing target is out of class (Theorem~\ref{thm:nonneg-nmf}), and identifiability is bought at the cost of excluding negative interactions. One measurement subtlety, exploited throughout: $\mathrm{cond}(\Sigma_f\Sigma_g)$ is gauge invariant, its eigenvalues equal the $\sigma_k^2$ by Lemma~\ref{lem:gauge-inv}, so it cannot separate normalizations and is not reported as a diagnostic.

\begin{figure}[t]
\centering
\includegraphics[width=\textwidth]{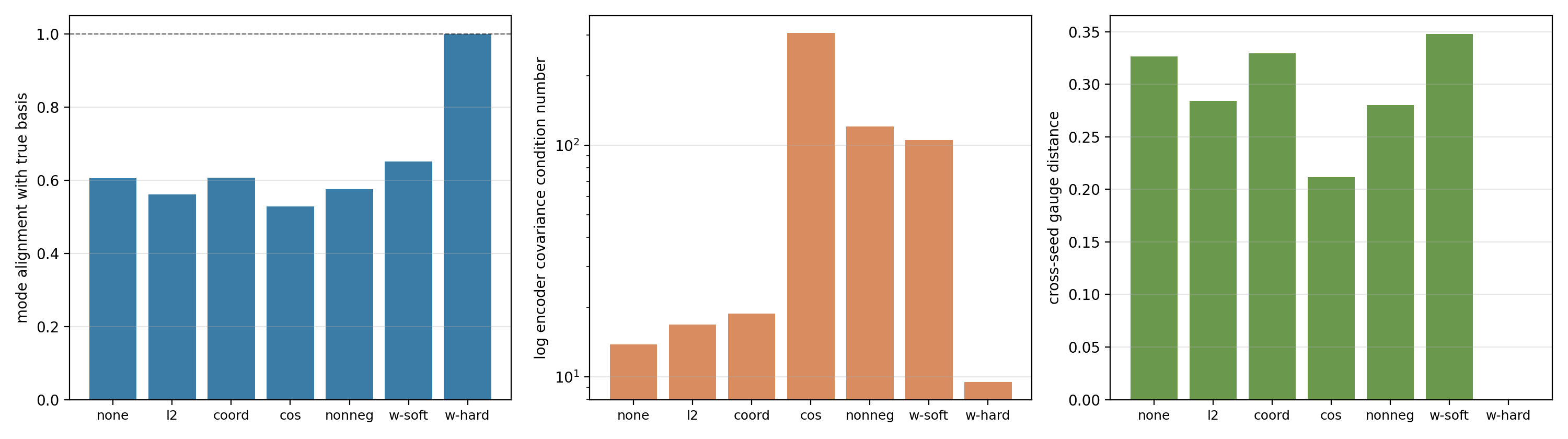}
\caption{The seven normalization schemes of Table~\ref{tab:zoo}, means over four seeds. Left: mode alignment with the true basis (dashed line: $1$). Middle: the conditioning diagnostic $\max(\mathrm{cond}\,\Sigma_f,\mathrm{cond}\,\Sigma_g)$ on a log scale, whose floor is $\sigma_1^2/\sigma_d^2\approx 9.5$. Right: cross-seed gauge distance (the short tick marks the exact zero of hard whitening). Abbreviations: l2 = L2 normalization, coord = per-coordinate standardization, cos = cosine normalization, nonneg = two-sided nonnegativity, w-soft = soft whitening, w-hard = hard whitening.}
\label{fig:zoo}
\end{figure}

\paragraph{The gap collapse.}
The gap sweeps use the finite-sample whitening estimator of Appendix~\ref{app:finite} directly, empirical second-moment whitening followed by eigendecomposition, rather than a trained optimizer: with Legendre features the target is represented exactly, and the training identification error has a sharp threshold around $n\approx 300$ that would hide the $1/\Delta$ law, whereas the perturbation mechanism of the estimator exposes it cleanly. Table~\ref{tab:gap} in Section~\ref{sec:exp-synth} reports the two sweeps. The eight points collapse onto a single constant, $\mathrm{err}\cdot\Delta\cdot\sqrt{n}\approx 2.48$ with coefficient of variation $0.12$: gap and sample size enter the identification error only through the product, the concrete face of Corollaries~\ref{cor:whiten-sc} and \ref{cor:gap-cond}. Figure~\ref{fig:gap} shows the collapse.

\begin{figure}[t]
\centering
\includegraphics[width=\textwidth]{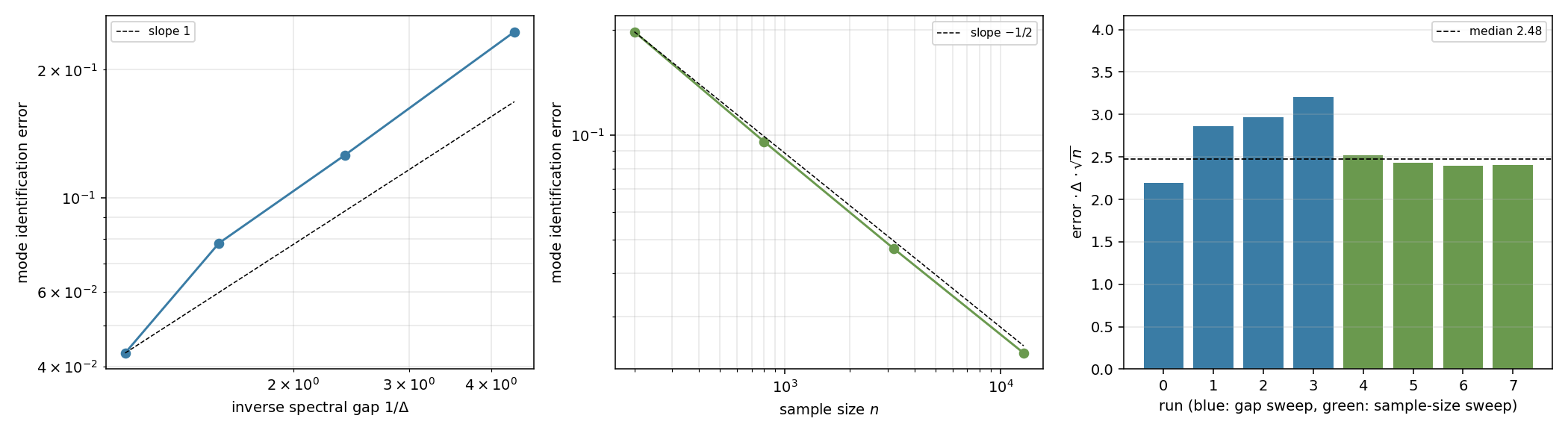}
\caption{The gap collapse. Mode identification error of the finite-sample whitening estimator: against the inverse spectral gap $1/\Delta$ at fixed $n$ (left) and against sample size $n$ at fixed $\Delta$ (center). Right: the product error$\,\cdot\Delta\cdot\sqrt{n}$ per run across both sweeps, near-constant at $\approx 2.48$.}
\label{fig:gap}
\end{figure}

\paragraph{The sample-complexity rate law.}
The linear setting is rank-$d$ matrix sensing, $y_i=u_i^{\top}\Theta^{\star}v_i+\varepsilon_i$ with $\rank\Theta^{\star}=d$, and Theorem~\ref{thm:minimax} predicts $\E\|\hat\Theta-\Theta^{\star}\|_F^2\asymp\sigma_{\varepsilon}^2 d(p+q)/n$. We sweep each variable with the others fixed and fit log-log slopes: $n$ at $p=q=4$, $d=2$ gives $-1.03$ (theory $-1$); $d$ at $p=q=16$, $n=2048$ gives $+0.89$ (theory $+1$); $p+q$ at $n=10^4$, $d=4$ gives $+1.24$ (theory $+1$). The excess of the last slope is accounted for by the Marchenko-Pastur finite-sample correction of the Gaussian design: measured ratios to $\sigma_{\varepsilon}^2 d(2p-d)/n$ are $1.00,1.05,0.98,1.09,1.13$ across $p=q=8,\dots,32$, tracking the correction factor $pq/(n-pq)$, which grows from $1.01$ to $1.11$ over the same range. Gradient descent on the factored pair $(U,V)$, the estimator class of this paper, is warm-started from the truncated least-squares factors and reproduces the least-squares slope, $-1.01$, on the same grid. The sweep rejects the loose general bound of Remark~\ref{rem:conservative}, which scales as $d^{3/2}\sqrt{(p+q)/n}$ and predicts slope $-1/2$ in $n$: the $1/n$ dependence on samples is real (Figure~\ref{fig:sampcomp}).

\paragraph{The negative results.}
Tables~\ref{tab:eq} and~\ref{tab:band} in Section~\ref{sec:exp-synth} give the two regimes; Figure~\ref{fig:negative} shows both graphically. On the boolean equality kernel $F=\one[u=v]$ on $\{\pm1\}^m$, the embedding rank needed for $10\%$ error tracks $d\approx 0.9\cdot 2^m$, exponential in $m$, while the early-interaction threshold on $\inner{u}{v}$ represents $F$ exactly with $2m+1$ parameters (Theorems~\ref{thm:flat-floor} and \ref{thm:early-escape}). On the narrowband periodic-Gaussian kernel $F_h(u,v)=\rho_h(u-v)$ of bandwidth $h$ on the circle, the required rank $d_{\varepsilon}$ grows as $\Theta(1/h)$ as the kernel approaches the diagonal (Remark~\ref{rem:smooth}).

\begin{figure}[t]
\centering
\includegraphics[width=\textwidth]{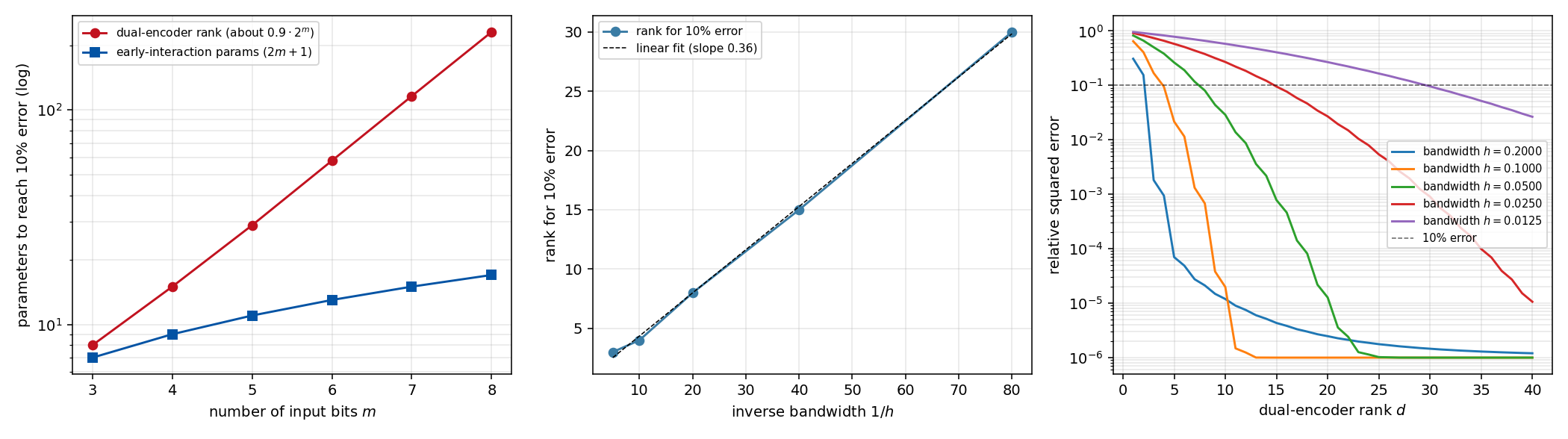}
\caption{The negative results. Left: on the boolean equality kernel, parameters needed to reach $10\%$ relative error against the number of input bits $m$, for the dual-encoder head (rank $d$) and the early-interaction threshold. Center: for the narrowband periodic-Gaussian kernel, the dual-encoder rank needed for $10\%$ error against the inverse bandwidth $1/h$, with a linear fit. Right: relative squared error of the dual-encoder head against its rank $d$, one curve per bandwidth $h$, with the $10\%$ line marked.}
\label{fig:negative}
\end{figure}

\begin{figure}[t]
\centering
\includegraphics[width=\textwidth]{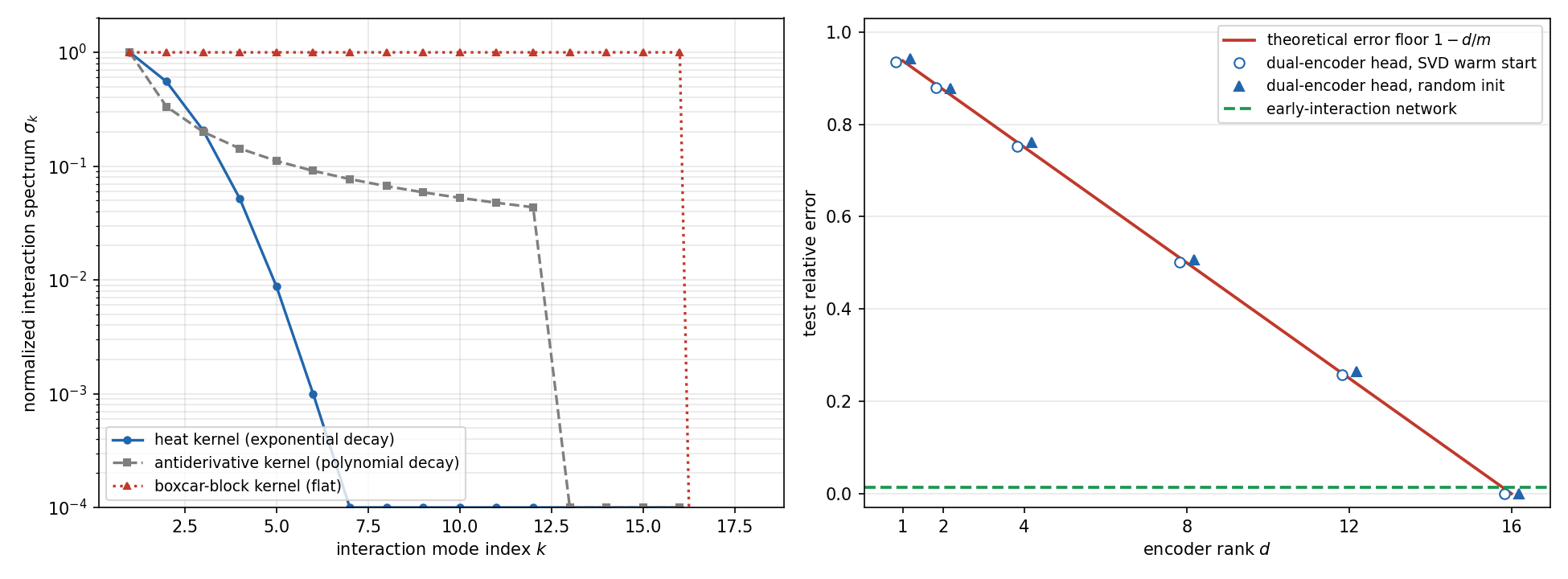}
\caption{The flat-spectrum error floor on a trained DeepONet head (Theorems~\ref{thm:flat-floor} and \ref{thm:early-escape}). Left: normalized interaction spectra $\sigma_k$ for three operators, from exponentially decaying (heat) to polynomially decaying (antiderivative) to exactly flat (boxcar block operator). Right: test relative error of a rank-$d$ dual-encoder head versus $d$, from an SVD warm start and from random initialization, against the predicted floor $1-d/m$; the dashed line is an early-interaction MLP baseline.}
\label{fig:usability}
\end{figure}

\subsection{DeepONet}
\label{app:exp-deeponet}

This study carries the identifiability results of Section~\ref{sec:gauge} onto a real dual-encoder architecture. We train branch-trunk DeepONet heads end to end on four operators and ask three questions, one per paragraph below: whether the rank-selection error curve tracks the decay of the operator spectrum (Corollary~\ref{cor:opt-rank}); whether post-hoc whitening, applied after training without touching the loss, lifts the learned trunk basis to the analytic operator eigenmodes and collapses the cross-seed gauge distance (Theorem~\ref{thm:whiten-id}); and whether the whitened gauge improves out-of-distribution error (Corollary~\ref{cor:gap-cond}). Two operators, the heat semigroup and the Volterra antiderivative, have closed-form spectra and provide exact ground truth for the recovered basis; two, viscous Burgers and Darcy flow, have only empirical spectral proxies and are used for the qualitative claims. The main-text Section~\ref{sec:exp} reports the conclusions; the tables and figures below give the full numbers.

\paragraph{Architecture and training.}
The branch and trunk networks are MLPs of width $128$ and depth $3$, outputting embeddings of dimension $d=32$; the rank-selection sweep of Table~\ref{tab:rankselect} varies $d$ from 2 to 64. Training runs 300 epochs with Adam at learning rate $10^{-3}$ and cosine annealing, batch size 256, and five seeds per configuration. Inputs are Gaussian random fields. The four operators are the heat semigroup $e^{t\Delta}$ (analytic singular values $\sigma_k=e^{-(k\pi)^2t}$, exponential), the antiderivative or Volterra operator ($\sigma_k=1/((k-\tfrac12)\pi)$, polynomial), and viscous Burgers and Darcy flow, whose spectra are empirical proxies, labeled as such. Whitening is realized in two stages: the stable soft-whitening penalty during training, the relaxation of Theorem~\ref{thm:soft-relax} and closely related to the Barlow Twins objective, and hard whitening post hoc on the active rank-$r$ subspace. The two-stage design is forced by the data: the interaction spectra decay fast, leaving an effective rank of 3 to 8 against an embedding width of 32, and in-forward hard whitening inverts about 25 near-zero covariance eigenvalues and backpropagates through a degenerate eigendecomposition, exactly the small-gap failure predicted by the block-stability analysis of Corollary~\ref{cor:block} in Appendix~\ref{app:finite}. Because whitening leaves $\inner{b}{t}$ unchanged, applying it after training is exact, and projecting to the active subspace first avoids inverting dead directions. The empirical interaction spectrum is estimated as the square roots of the eigenvalues of $\Sigma_b^{1/2}\Sigma_t\Sigma_b^{1/2}$, computed with $\mathrm{eigvalsh}$; this avoids the numerically unstable SVD of the ill-conditioned product $\Sigma_b^{1/2}\Sigma_t^{1/2}$ and agrees with the direct computation to $10^{-9}$ on well-conditioned instances.

\paragraph{Spectra and rank selection.}
The normalized empirical spectra reproduce the dichotomy qualitatively, heat $1,0.22,0.028,\dots$ and antiderivative $1,0.14,0.042,\dots$. Because the network allocates no capacity to near-null modes, the learned spectra decay faster than the analytic ones, so we claim recovery of the dominant modes and of the decay type, not a mode-by-mode match. Table~\ref{tab:rankselect} gives the rank sweep. The curves are flat past $d\approx 8$ and the argmin is dominated by noise, so we do not report a single optimal rank; the honest quantity is the drop from $d=4$ to $d=8$, $5\%$ for heat (early saturation, small optimal rank) versus $31\%$ for antiderivative (more modes needed), consistent with Corollary~\ref{cor:opt-rank}. This is qualitative support, and the figure is labeled accordingly.

\begin{table}[t]
\centering
\small
\setlength{\tabcolsep}{4pt}
\caption{Rank selection, test relative $L^2$ error versus trunk width $d$. The drop from $d=4$ to $d=8$ is the claimed quantity: $5\%$ for heat versus $31\%$ for antiderivative (Corollary~\ref{cor:opt-rank}).}
\label{tab:rankselect}
\begin{tabular}{@{}l r r r r r r r@{}}
\toprule
Operator & $d=2$ & $d=4$ & $d=8$ & $d=16$ & $d=32$ & $d=64$ & Drop $4\to8$ \\
\midrule
heat & .070 & .056 & .053 & .052 & .055 & .056 & $-5\%$ \\
antiderivative & .106 & .081 & .056 & .055 & .070 & .060 & $-31\%$ \\
burgers & .237 & .135 & .134 & .133 & .133 & .136 & flat \\
darcy & .167 & .150 & .149 & .153 & .158 & .169 & flat \\
\bottomrule
\end{tabular}
\end{table}

\paragraph{Basis recovery.}
Table~\ref{tab:basis} gives the full per-gauge numbers behind Figure~\ref{fig:basis}. Alignment is the mean over modes of $|\inner{\hat t_k}{\beta_k}|$ after optimal permutation and sign matching against the analytic left singular functions, before and after the post-hoc whitening gauge fix; cross-seed is the mean pairwise gauge distance between seeds. Whitening lifts the loose gauges from $0.5$ to $0.6$ to $0.78$ to $0.94$ and collapses the cross-seed distance by roughly an order of magnitude; for heat the recovered modes are the analytic Fourier basis $\sqrt{2}\sin(k\pi y)$, a hard ground truth. On the nonlinear benchmarks, where no analytic basis exists, whitening still tightens the cross-seed distance for Burgers, $0.30\to0.08$; Darcy is the weakest case, $0.56\to0.33$, and its proxy spectrum is the noisiest, which the figure notes.

\begin{figure}[t]
\centering
\includegraphics[width=\textwidth]{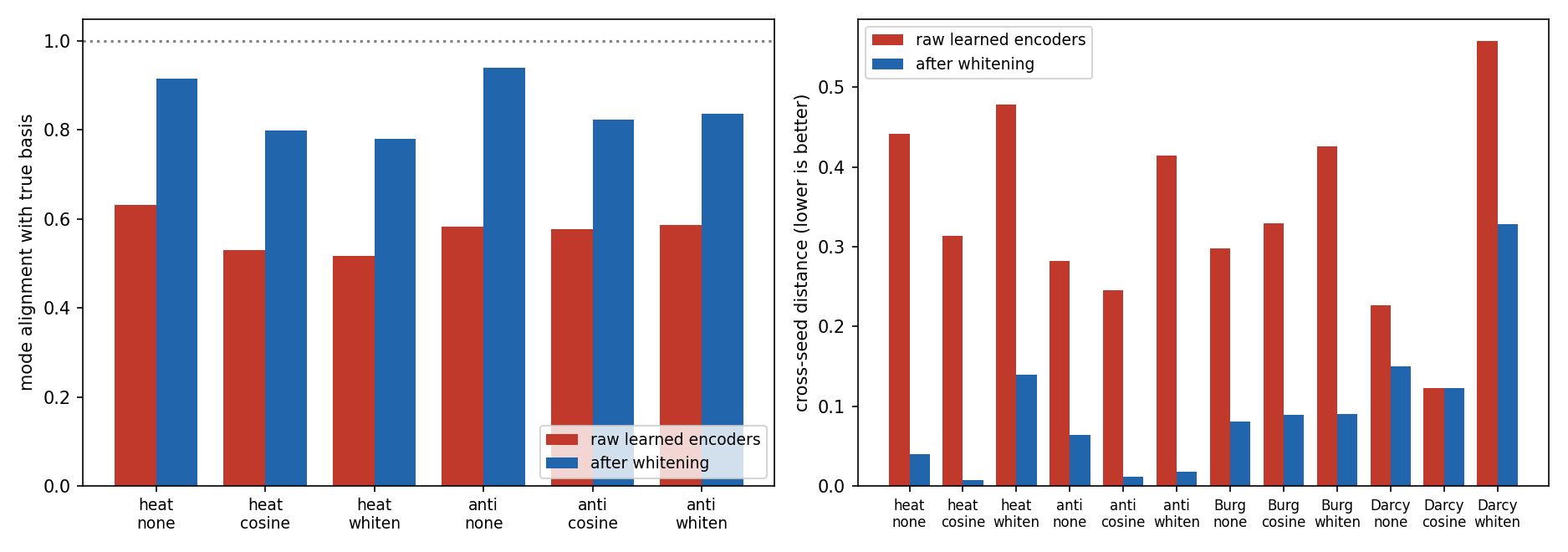}
\caption{Trunk-encoder basis quality before and after post-hoc whitening (Theorem~\ref{thm:whiten-id}). Each group on the horizontal axis corresponds to one operator (heat, anti = antiderivative, Burg = Burgers, Darcy) trained under one gauge constraint (none, cosine, or whiten = soft whitening). Left: alignment of the learned trunk basis with the analytic operator eigenmodes ($1$ is exact). Right: distance between bases learned from different random seeds (lower indicates greater stability).}
\label{fig:basis}
\end{figure}

\begin{table}[t]
\centering
\small
\setlength{\tabcolsep}{3pt}
\caption{Basis recovery per gauge, five seeds. Alignment: mean mode alignment against the analytic left singular functions, raw and after post-hoc whitening. Cross-seed: mean pairwise gauge distance, raw and after whitening. The effect is Theorem~\ref{thm:whiten-id} on a trained model.}
\label{tab:basis}
\begin{tabular}{@{}>{\raggedright\arraybackslash}p{3.2cm} r r r r@{}}
\toprule
Operator / gauge & Align & Whitened align & Seed dist & Whitened dist \\
\midrule
heat / none & 0.632 & 0.915 & 0.442 & 0.039 \\
heat / cosine & 0.529 & 0.799 & 0.313 & 0.007 \\
heat / whiten\_soft & 0.517 & 0.780 & 0.479 & 0.140 \\
antiderivative / none & 0.582 & 0.940 & 0.282 & 0.064 \\
antiderivative / cosine & 0.576 & 0.823 & 0.245 & 0.011 \\
antiderivative / whiten\_soft & 0.587 & 0.837 & 0.414 & 0.018 \\
\bottomrule
\end{tabular}
\end{table}

\paragraph{The payoff of gauge-fixing.}
Table~\ref{tab:payoff} gives the in-distribution and out-of-distribution relative $L^2$ errors of the none and soft-whitening gauges; the out-of-distribution test set draws inputs from a different smoothness regime of the same operator family. Soft whitening attains the lowest error in all eight cells, with the out-of-distribution gain largest, $2.4\times$ for heat. Two caveats are recorded with the table. First, soft whitening does not reduce seed-to-seed error variance; if anything it increases it, heat standard deviation $0.0005\to0.0036$ and Darcy $0.0017\to0.0181$, so the stability it confers is in the identified basis of Table~\ref{tab:basis}, not in the error. Second, the soft-whitening training loss carries the decorrelation penalty and its curve is not directly comparable to the unpenalized gauges, so we make no claim about convergence speed.

\begin{table}[t]
\centering
\small
\setlength{\tabcolsep}{4pt}
\caption{The payoff of gauge-fixing, five seeds: relative $L^2$ error in distribution and out of distribution, none versus soft whitening. Soft whitening is best in all cells; the out-of-distribution gain is the headline (Corollary~\ref{cor:gap-cond}).}
\label{tab:payoff}
\begin{tabular}{@{}>{\raggedright\arraybackslash}p{2.2cm} r r r r r@{}}
\toprule
Operator & none, in-dist. & whiten, in-dist. & none, OOD & whiten, OOD & OOD gain \\
\midrule
heat & 0.051 & 0.019 & 0.038 & 0.016 & $2.4\times$ \\
antiderivative & 0.057 & 0.039 & 0.046 & 0.031 & $1.5\times$ \\
burgers & 0.122 & 0.112 & 0.113 & 0.103 & $1.1\times$ \\
darcy & 0.151 & 0.137 & 0.142 & 0.120 & $1.2\times$ \\
\bottomrule
\end{tabular}
\end{table}

\paragraph{The flat-spectrum anchor.}
The block-boxcar operator $G(a)(y)=c_{j(y)}$, the output is the block mean of the input field at the block containing $y$, with $m=16$ blocks, has an exactly flat interaction spectrum: sixteen equal singular values, the seventeenth at $10^{-16}$. Table~\ref{tab:boxcar} gives the trained relative error of the dual-encoder head with linear encoders, from both the SVD warm start and random initialization, against the floor $1-d/m$ of Theorem~\ref{thm:flat-floor}; the worst deviation is $0.014$, the finite-sample excess, and gradient descent reaches the floor from scratch. The early-interaction comparator, a ReLU MLP on $\mathrm{concat}(a,\mathrm{onehot}(y))$ with $12.4$k parameters, four times the budget of the $d=16$ head, reaches $0.013\pm 0.001$ on the same test set (Theorem~\ref{thm:early-escape}). One caveat belongs in the setup: with only 8k training samples the comparator memorizes, training MSE $4\times10^{-5}$ and test error $1.7$; at 60k samples, where parameters are far fewer than samples, memorization is impossible and the escape shown in Figure~\ref{fig:usability} is an honest generalization result. The comparison deliberately pits the floor against a well-trained early model.

\begin{table}[t]
\centering
\small
\setlength{\tabcolsep}{4pt}
\caption{The flat-spectrum anchor: trained relative error of a linear-encoder head on the block-boxcar operator versus the floor $1-d/m$ of Theorem~\ref{thm:flat-floor}$^,$\ref{thm:early-escape}. The two initializations agree, and the error hugs the floor from either start.}
\label{tab:boxcar}
\begin{tabular}{@{}r r r@{}}
\toprule
$d$ & Trained error & Floor $1-d/m$ \\
\midrule
1 & 0.936 & 0.938 \\
2 & 0.879 & 0.875 \\
4 & 0.752 & 0.750 \\
8 & 0.502 & 0.500 \\
12 & 0.258 & 0.250 \\
16 & 0.000 & 0.000 \\
\bottomrule
\end{tabular}
\end{table}

\subsection{CLIP}
\label{app:exp-clip}

This study tests the gauge view on heterogeneous modalities and on pretrained models that never saw each other. It proceeds in two stages. The white-box synthetic kernel supplies exact ground truth for the mechanism: under linear encoders and white inputs the interaction spectrum equals the singular values of $W_f^{\top}W_g$, which is gauge invariant, and two independently trained encoders are two random gauges of the same target; on this kernel we verify that the spectrum is invariant and equals the truth, that per-dimension concept probes do not transfer while the canonical subspace does, and that whitening plus gap recovers the text modes exactly. The real CLIP pairs then answer whether the same statements hold for contrastive models trained in the wild: is the interaction spectrum shared across independently pretrained models, are per-dimension probes not, and is the residual relationship a single rotation that whitening removes, exposing interpretable concept axes (Theorem~\ref{thm:cosine-od})? The main-text Section~\ref{sec:exp} reports the conclusions; the tables and figures below give the full numbers.

\paragraph{The synthetic multimodal kernel.}
The white-box counterpart is $F^{\star}(u,v)=\sum_{k}\rho_k\,(p_k^{\top}u)(q_k^{\top}v)$ with $u\sim\mathcal{N}(0,I_p)$, $v\sim\mathcal{N}(0,I_q)$, orthogonal canonical directions $P,Q$, and distinct, gapped canonical correlations $\rho=(0.9,0.75,0.6,0.45,0.3,0.15)$. The key identity: under linear encoders and white inputs the interaction spectrum equals the singular values of $B=W_f^{\top}W_g$, which is explicitly gauge invariant, $f\to Af$, $g\to A^{-\top}g$ leaves $W_f^{\top}W_g$ unchanged, and two independently trained encoders are two random gauges $A_1,A_2$. Table~\ref{tab:clip-synth} collects the numbers: the spectrum is invariant across gauges to $3\times10^{-16}$ and equals the true $\rho$ to $4\times10^{-16}$; per-dimension concept probes transfer at $0.43$, the best-matched axis reaching only $0.65$, while the canonical subspace aligns to $1.0000$; whitening plus gap recovers the true text modes at $|\cos|=1.000$, against $0.554$ raw and $0.581$ under cosine, which is stuck at the residual $\Od$; and cross-encoder frame stability rises from $0.684$ to $1.000$. The distance between the $0.58$ of cosine and the $1.00$ of whitening is, numerically, the contribution of whitening beyond the decorrelation that self-supervised methods already perform, precisely the quantity Theorem~\ref{thm:whiten-id} isolates.

\begin{table}[t]
\centering
\small
\setlength{\tabcolsep}{4pt}
\caption{Synthetic multimodal kernel, ground truth known. The interaction spectrum is the unique gauge-invariant content; whitening plus gap pins the modes, cosine does not.}
\label{tab:clip-synth}
\begin{tabular}{@{}>{\raggedright\arraybackslash}p{6.2cm} r@{}}
\toprule
Quantity & Value \\
\midrule
Spectrum invariant across gauges & $3\times10^{-16}$ \\
Spectrum equals true $\rho$ & $4\times10^{-16}$ \\
Per-dimension probe transfer & $0.43$ (best axis $0.65$) \\
Canonical subspace alignment & $1.0000$ \\
Text mode recovery, raw & $0.554$ \\
Text mode recovery, cosine (isotropization) & $0.581$ \\
Text mode recovery, whitening plus gap & $1.000$ \\
Cross-encoder frame stability, raw $\to$ whitened & $0.684\to1.000$ \\
\bottomrule
\end{tabular}
\end{table}

\paragraph{Real CLIP.}
Two backbone pairs, each two independent gauges of one similarity function: ViT-B/32 (OpenAI) versus ViT-B/32 (LAION-2B), the same architecture with different pretraining data, and ViT-B/32 (OpenAI) versus RN50 (OpenAI), different architectures. Image features come from CIFAR-100 test images, 49 fine classes spanning animals, vehicles, plants, and household objects, 64 images per class; text features are prompt-ensembled class embeddings plus attribute contrast directions; both towers are projected to a common 32-dimensional working space, the image PCA top-$d$ directions. Whitening is realized via canonical correlation analysis with shrinkage covariance, the finite-sample counterpart of the Theorem~\ref{thm:whiten-id} gauge fix.

Table~\ref{tab:clip-m1} and Figure~\ref{fig:clipspectrum} give the spectrum comparison across backbone pairs: near-identical spectra, agreement best at the top. Per-dimension probe transfer is $0.28$ without alignment and $0.92$ after fitting a single global orthogonal map: the two embeddings are related by one $\Od$ rotation, the residual group Theorem~\ref{thm:cosine-od} predicts for the cosine gauge, and no coordinatewise repair exists short of discovering that rotation. Table~\ref{tab:clip-m2} gives the group-signal regression of true labels on the top-6 mode activations: the whitened span contains the animate and natural axes, the cosine span contains only the animate axis, and per-mode statistics show why, whitening concentrates the animate signal in a single mode ($|t|=3.1$) while cosine spreads it across several. Figure~\ref{fig:clip-heatmap} shows the canonical text modes as concept axes: mode 1 separates animals (leopard, bear, wolf, camel) from flowers and fruit, mode 2 is the vehicle axis (bus, streetcar, pickup truck, train), mode 3 separates plants from household objects, and mode 5 the food axis (orange, apple, pear, sweet pepper); the cosine principal coordinates are semantic mixtures by comparison. Cross-backbone frame stability, matching modes by their backbone-independent class-signature profiles, favors whitening in both pairs, $0.77$ versus $0.72$ for the same-architecture pair and $0.84$ versus $0.62$ across architectures.

\begin{table}[t]
\centering
\small
\setlength{\tabcolsep}{4pt}
\caption{Interaction spectra across independently pretrained CLIP models. Pearson correlation and normalized L1 difference of the two spectra; per-dimension probe transfer raw and after fitting the single residual $\Od$ rotation (Theorem~\ref{thm:cosine-od}).}
\label{tab:clip-m1}
\begin{tabular}{@{}>{\raggedright\arraybackslash}p{4.4cm} r r@{}}
\toprule
Backbone pair & Spectrum Pearson & Norm. L1 \\
\midrule
ViT-B/32 (OpenAI) vs ViT-B/32 (LAION) & 0.99 & 0.065 \\
ViT-B/32 (OpenAI) vs RN50 (OpenAI) & 0.98 & 0.096 \\
\bottomrule
\end{tabular}
\end{table}

\begin{figure}[t]
\centering
\includegraphics[width=\textwidth]{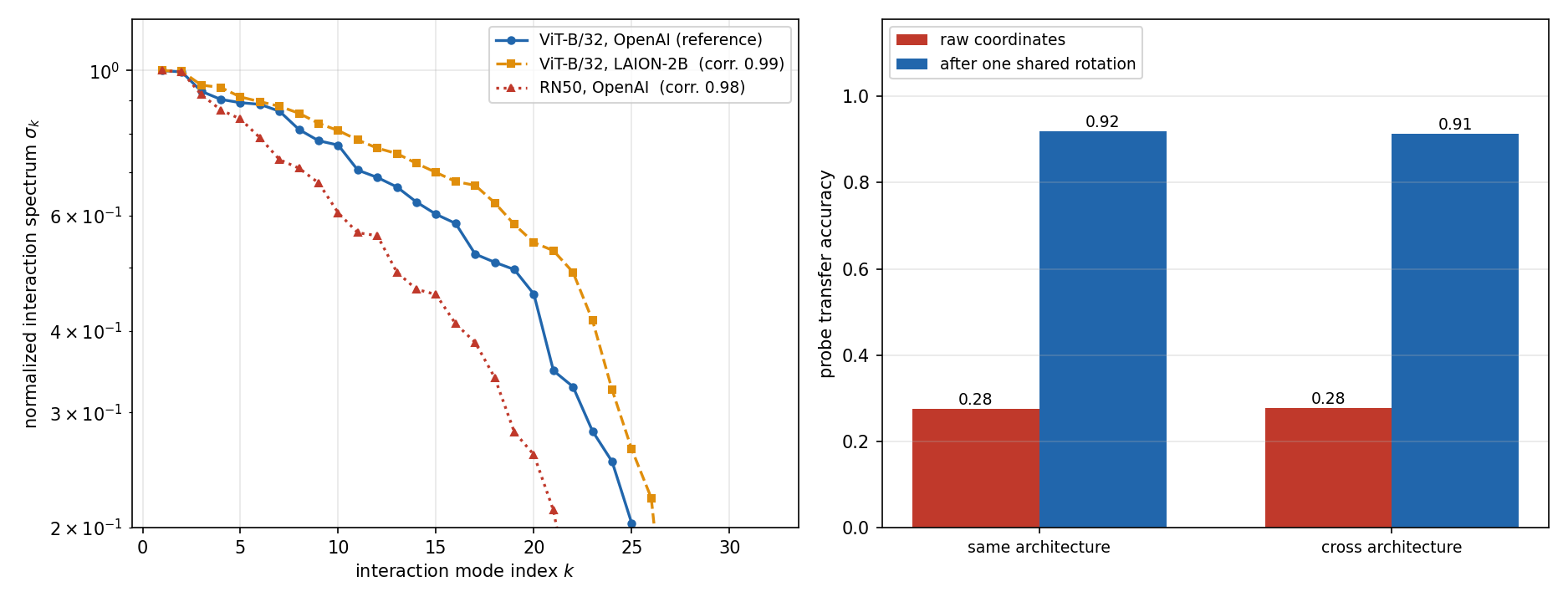}
\caption{Shared structure across independently trained CLIP models (Theorem~\ref{thm:cosine-od}). Left: normalized interaction spectra \(\sigma_k\) for three CLIP encoders (a ViT-B/32 reference against a LAION-trained ViT-B/32 and an RN50), with Pearson correlations to the reference in the legend. Right: accuracy of a linear concept probe trained on one model and applied to the other, using raw coordinates and after fitting a single shared rotation, for same-architecture and cross-architecture pairs.}
\label{fig:clipspectrum}
\end{figure}

\begin{table}[t!]
\centering
\small
\setlength{\tabcolsep}{4pt}
\caption{Group-signal concentration on the reference backbone: $R^2$ of regressing true group labels on the top-6 mode activations, whitened canonical modes versus cosine principal coordinates. Cross-backbone frame stability is reported in the text.}
\label{tab:clip-m2}
\begin{tabular}{@{}>{\raggedright\arraybackslash}p{3.4cm} r r@{}}
\toprule
Group split & Whitened & Cosine \\
\midrule
animate / rest & 0.92 & 0.88 \\
natural / manmade & 0.90 & 0.48 \\
\bottomrule
\end{tabular}
\end{table}

\begin{figure}[t!]
\centering
\includegraphics[width=0.8\textwidth]{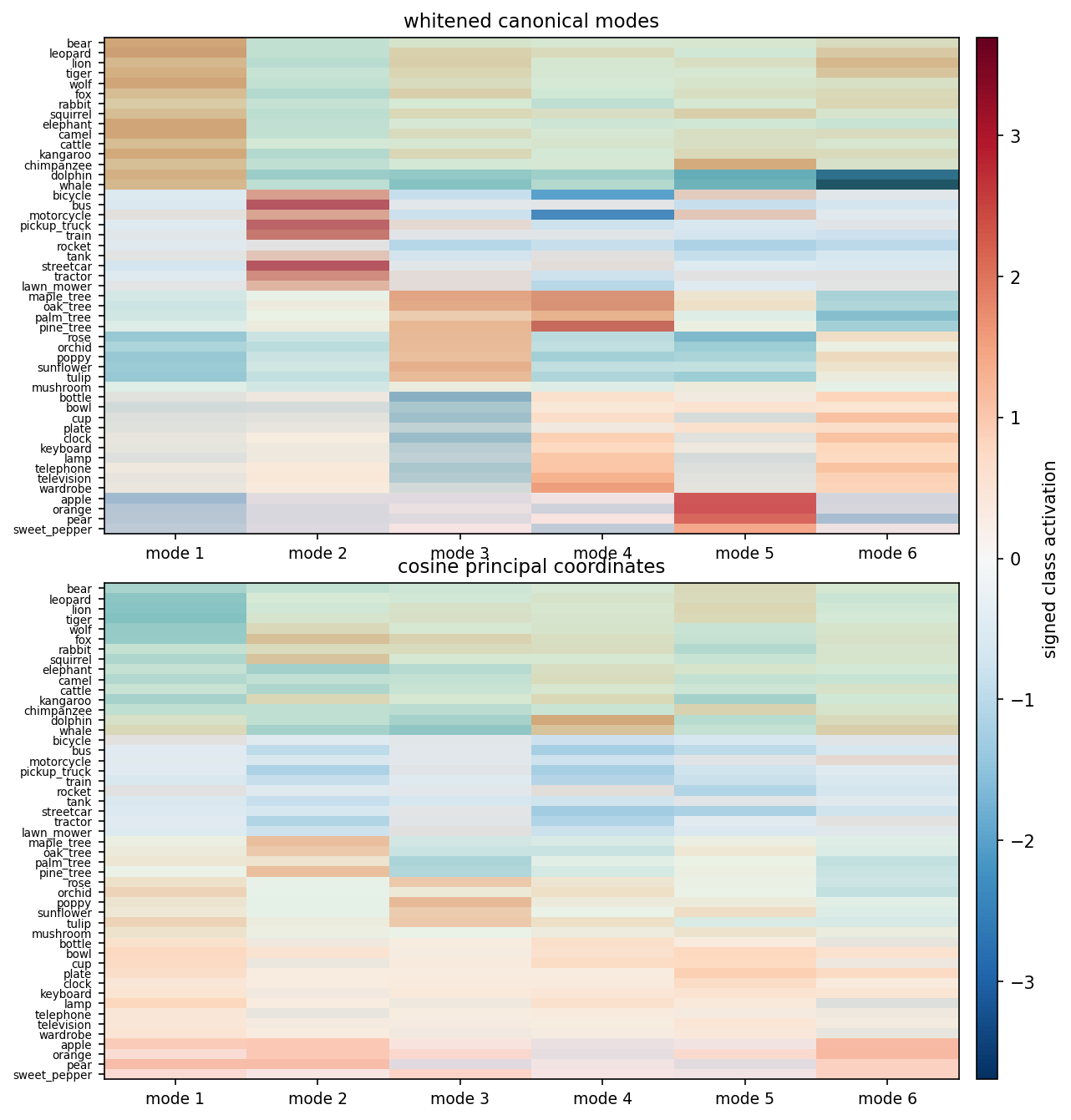}
\caption{Whitened canonical text modes (rows) over the 49 CIFAR-100 classes (columns) of the reference backbone: modes read as concept axes (animals, vehicles, plants, food), while the cosine principal coordinates are semantic mixtures.}
\label{fig:clip-heatmap}
\end{figure}

\paragraph{What whitening buys for CLIP in practice.}
None of the steps above requires retraining: whitening is a post-hoc second-moment transform, it leaves the inner-product scores unchanged, and at inference time it costs one fixed linear map. It can therefore be applied as an analysis layer on any deployed CLIP model, and the results above suggest three concrete payoffs. First, model alignment: because independently pretrained pairs are related by a single residual rotation, whitening both models into the shared canonical frame makes their coordinates comparable dimension by dimension, a prerequisite for model merging, distillation, and cross-model probing that raw cosine coordinates do not offer. Second, interpretability with known limits: the whitened axes read as concept axes (animals, vehicles, plants, and food in the CIFAR-100 probe), and the spectral gap quantifies in advance how much interpretability any basis can carry, so an analyst knows where the axes are trustworthy and where the small-gap failure of Appendix~\ref{app:finite} leaves them partially mixed. Third, a spectral summary of the model itself: the interaction spectrum is a gauge-invariant description of what a CLIP pair can discriminate, and its decay predicts the effective interaction rank (Corollary~\ref{cor:opt-rank}) and hence how many concept axes are worth reading out at all. A natural next step is to use the whitened canonical frame as the initialization or constraint for fine-tuning, so that new axes are learned in coordinates that are already interpretable and transferable across models.

\paragraph{Summary.}
The synthetic study establishes the three claims where ground truth is exact: whitening uniquely fixes the gauge, the gap controls identification through the product $\Delta\sqrt{n}$ (collapse constant $2.48$), and flat spectra force exponential late-interaction cost. DeepONet carries identifiability onto a real architecture, recovering the analytic eigenbasis with $0.92$ to $0.94$ alignment and an out-of-distribution payoff up to $2.4\times$; CLIP extends the claim to heterogeneous modalities, the spectrum shared across independently pretrained models while per-dimension probes are not. On a flat spectrum the head is floored at $1-d/m$ while an early-interaction model reaches $0.013$: the interaction spectrum decides whether a multiplicative dual-encoder head is the right tool.

\section{Index of results}
\label{app:index}
Table~\ref{tab:index} collects the paper's numbered results with a one-line statement of what each establishes and where it is stated or proved. The results are grouped by the four questions the paper answers, with the finite-sample lemmas of Appendix~\ref{app:finite} last; within each group the numbering follows the order of appearance, and the table doubles as a reading map for the appendices. Remarks and the technical lemmas internal to proofs are not indexed separately; they are cited where they are used.

{\small
\setlength{\tabcolsep}{4pt}
\begin{longtable}{@{}l p{0.55\textwidth} l@{}}
\caption{Index of the paper's numbered results.}
\label{tab:index}\\
\toprule
Result & What it establishes & Where \\
\midrule
\endfirsthead
\multicolumn{3}{@{}l}{\itshape Index of results (continued)}\\
\toprule
Result & What it establishes & Where \\
\midrule
\endhead
\bottomrule
\endlastfoot
\multicolumn{3}{@{}l}{\itshape Unification and approximation (Section~\ref{sec:unif})} \\
\midrule
Definition~\ref{def:spectrum} & Interaction spectrum, rank, and modes as the singular data of the target & Section~\ref{sec:unif} \\
Theorem~\ref{thm:err-split} & Approximation error decomposes into a truncation term and an encoder-realization term & Section~\ref{sec:unif} \\
Assumption~\ref{asm:realizability} & Encoder classes contain approximations of the scaled interaction modes & Appendix~\ref{app:proofs-sec2} \\
Proposition~\ref{prop:compute} & Evaluating a rank-$d$ head on all pairs costs $O((n+m)\,C_{\mathrm{net}}+nm\,d)$ & Appendix~\ref{app:proofs-sec2} \\
Lemma~\ref{lem:rank-equiv} & Representability of $F$ by a rank-$d$ head iff $\rank(T_F)\le d$ & Appendix~\ref{app:proofs-sec2} \\
Corollary~\ref{cor:trunc-lower} & The truncation term is the best possible error of any rank-$d$ head & Appendix~\ref{app:proofs-sec2} \\
Theorem~\ref{thm:decay} & Target smoothness controls the decay of the truncation tail & Appendix~\ref{app:proofs-sec2} \\
Proposition~\ref{prop:separation} & The dual head escapes the curse of the joint input dimension & Appendix~\ref{app:proofs-sec2} \\
\midrule
\multicolumn{3}{@{}l}{\itshape Identifiability as gauge-fixing (Section~\ref{sec:gauge})} \\
\midrule
Definition~\ref{def:gauge} & Gauge action of $\GL_d$ on encoder pairs leaves the represented function unchanged & Section~\ref{sec:gauge} \\
Theorem~\ref{thm:whiten-id} & Under a spectral gap, whitening identifies the modes up to permutation and sign & Section~\ref{sec:gauge} \\
Proposition~\ref{prop:percoord} & Per-coordinate standardization pins the coordinate scales but no more & Appendix~\ref{app:proofs-sec3} \\
Lemma~\ref{lem:gauge-inv} & The eigenvalues of $\Sigma_f\Sigma_g$ are gauge invariants, equal to the $\sigma_k^2$ at any optimum & Appendix~\ref{app:proofs-sec3} \\
Proposition~\ref{prop:illposed} & Unfixed gauge renders the risk flat along the orbit and optimization ill-posed & Appendix~\ref{app:proofs-sec3} \\
Theorem~\ref{thm:cosine-od} & Cosine normalization leaves the full orthogonal group as residual gauge & Appendix~\ref{app:proofs-sec3} \\
Corollary~\ref{cor:cosine-uninterpretable} & Individual coordinates of a cosine-normalized embedding carry no intrinsic meaning & Appendix~\ref{app:proofs-sec3} \\
Theorem~\ref{thm:nonneg-nmf} & Two-sided nonnegativity gives NMF-style identifiability, at the cost of excluding negative modes & Appendix~\ref{app:proofs-sec3} \\
Corollary~\ref{cor:gap-cond} & The spectral gap controls conditioning, identifiability, and convergence simultaneously & Appendix~\ref{app:proofs-sec3} \\
Theorem~\ref{thm:soft-relax} & A quadratic whitening penalty inherits the identifiability of the hard constraint & Appendix~\ref{app:proofs-sec3} \\
\midrule
\multicolumn{3}{@{}l}{\itshape Estimation and rank selection (Section~\ref{sec:sample})} \\
\midrule
Lemma~\ref{lem:rademacher} & The Rademacher complexity of the inner-product class collapses to the sum of the encoder complexities & Appendix~\ref{app:proofs-sec4} \\
Corollary~\ref{cor:excess} & Excess risk of empirical risk minimization on the head class & Appendix~\ref{app:proofs-sec4} \\
Theorem~\ref{thm:minimax} & Linear heads are rank-$d$ matrix sensing with minimax rate $\sigma_{\varepsilon}^2 d(p+q)/n$ & Appendix~\ref{app:proofs-sec4} \\
Corollary~\ref{cor:opt-rank} & Optimal rank grows logarithmically in the sample budget for smooth targets & Appendix~\ref{app:proofs-sec4} \\
\midrule
\multicolumn{3}{@{}l}{\itshape The usability criterion (Section~\ref{sec:sample})} \\
\midrule
Theorem~\ref{thm:flat-floor} & A flat spectrum forces every rank-$d$ head to relative error $1-d/N$ & Appendix~\ref{app:proofs-sec4} \\
Theorem~\ref{thm:early-escape} & Early-interaction models escape the flat-spectrum floor & Appendix~\ref{app:proofs-sec4} \\
Theorem~\ref{thm:criterion} & Usability criterion: decide by the measured spectrum whether the head can succeed & Appendix~\ref{app:proofs-sec4} \\
\midrule
\multicolumn{3}{@{}l}{\itshape Finite-sample identification (Appendix~\ref{app:finite})} \\
\midrule
Lemma~\ref{lem:berstein} & Matrix Bernstein concentration for empirical second moments & Appendix~\ref{app:finite} \\
Lemma~\ref{lem:whiten-stab} & Stability of the empirical whitening map under moment perturbation & Appendix~\ref{app:finite} \\
Theorem~\ref{thm:finite-id} & Finite-sample mode identification with error of order $r/\Delta$ & Appendix~\ref{app:finite} \\
Corollary~\ref{cor:whiten-sc} & Sample complexity of the whitening estimator: $n\gtrsim B^4\log(d/\delta)/(\Delta^2\varepsilon^2)$ & Appendix~\ref{app:finite} \\
Corollary~\ref{cor:block} & Block stability of mode recovery under small gaps & Appendix~\ref{app:finite} \\
Definition~\ref{def:alpha-sep} & $\alpha$-approximate separability as a quantitative relaxation of NMF conditions & Appendix~\ref{app:finite} \\
Theorem~\ref{thm:nmf-degrad} & Approximately separable targets degrade identifiability gracefully & Appendix~\ref{app:finite} \\
Proposition~\ref{prop:ledger} & Identification and estimation errors combine into one ledger & Appendix~\ref{app:finite} \\
Proposition~\ref{prop:risk-mode} & Mode error as the bridge between risk and identified basis & Appendix~\ref{app:finite} \\
\end{longtable}
}

\end{document}